\documentclass[]{article}

\usepackage[utf8]{inputenc}
\usepackage{fullpage}
\usepackage{url}
\usepackage{subcaption}
\usepackage{graphicx}
\usepackage{xcolor}
\usepackage{soul}  
\usepackage{afterpage}

\usepackage{amsmath}
\usepackage{amssymb}
\usepackage{amsfonts}
\usepackage{amsthm}
\usepackage{algorithm}
\usepackage[]{algpseudocode}

\makeatletter
\def\blfootnote{\xdef\@thefnmark{}\@footnotetext}
\makeatother

\newtheorem{theorem}{Theorem}[]

\newtheorem{lemma}{Lemma}[section]
\newtheorem{corollary}[lemma]{Corollary}

\theoremstyle{definition}
\newtheorem{definition}{Definition}[]

\newcommand{\argmin}{\mathop{\rm argmin}}
\newcommand{\argmax}{\mathop{\rm argmax}}
\newcommand{\sigmax}{\sigma_{\max} }
\newcommand{\sigmin}{\sigma_{\min} }
\newcommand{\ssigmin}{\bar{\sigma}_{\min}}

\newcommand{\cond}{\mathbf{\kappa}}

\newcommand{\order}{\mathcal{O}}

\newcommand{\reals}{\mathbf{R}}
\newcommand{\naturals}{\mathbf{N}}
\newcommand{\calC}{\mathcal{C}}
\newcommand{\calA}{\mathcal{A}}
\newcommand{\calP}{\mathcal{P}}

\newcommand{\uniform}{\mathbf{Unif}}
\newcommand{\bern}{\mathbf{Bern}}
\newcommand{\normal}{\mathbf{Normal}}
\newcommand{\exponential}{\mathbf{Exponential}}

\newcommand{\convsum}{\sum_{\ell=1}^{L} W_\ell H S_{\ell-1}}
\newcommand{\permi}{{P(i)}}
\newcommand{\permj}{{P(j)}}

\newcommand{\fullrank}{\textbf{[Full Rank]}}
\newcommand{\separable}{\textbf{[Separable]}}
\newcommand{\noisebound}{\textbf{[Bounded noise]}}
\newcommand{\distbound}{\textbf{[Sequentially unique]}}

\definecolor{brightpink}{rgb}{1.0, 0.0, 0.5}

\newcommand{\vtwo}[1]{{\color{black}#1}}
\newcommand{\vtwominor}[1]{{\color{black}#1}}

\makeatletter
\renewcommand{\maketag@@@}[1]{\hbox{\m@th\normalsize\normalfont#1}}%
\makeatother

\begin{document}

\title{A Provably Correct and Robust Algorithm for \\ Convolutive Nonnegative Matrix Factorization\blfootnote{\textcopyright 2019 IEEE.  Personal use of this material is permitted.  Permission from IEEE must be obtained for all other uses, in any current or future media, including reprinting/republishing this material for advertising or promotional purposes, creating new collective works, for resale or redistribution to servers or lists, or reuse of any copyrighted component of this work in other works.}}

\author{Anthony~Degleris
\thanks{Department of Electrical Engineering, Stanford University, Stanford CA, USA. E-mail: degleris@stanford.edu.}
~and~Nicolas~Gillis
\thanks{Department of Mathematics and Operational Research,
Facult\'e Polytechnique, Universit\'e de Mons,
Rue de Houdain 9, 7000 Mons, Belgium. NG acknowledges the support by the Fonds de la Recherche Scientifique - FNRS and the Fonds Wetenschappelijk Onderzoek - Vlanderen (FWO) under EOS Project no O005318F-RG47, and by the European Research Council (ERC starting grant no 679515).
E-mail: nicolas.gillis@umons.ac.be.}
}

\maketitle
\begin{abstract}
    In this paper, we propose a provably correct algorithm for convolutive nonnegative matrix factorization (CNMF) under separability assumptions.
    CNMF is a convolutive variant of nonnegative matrix factorization (NMF), which functions as an NMF with additional sequential structure.
    This model is useful in a number of applications, such as audio source separation and neural sequence identification.
    While a number of heuristic algorithms have been proposed to solve CNMF, to the best of our knowledge no provably correct algorithms have been developed.
    We present an algorithm that takes advantage of the NMF model underlying CNMF and exploits existing algorithms for separable NMF to provably find
    \vtwo{the unique solution (up to permutation and scaling)} under \vtwominor{separability-like} conditions.
    Our approach guarantees the solution in low noise settings, and runs in polynomial time.
    We illustrate its effectiveness on synthetic datasets, and on a singing bird audio sequence.
\end{abstract}



\section{Introduction}
\label{sec: intro}

Nonnegative matrix factorization (NMF) is a standard unsupervised learning technique for analyzing large datasets.
Given an $N \times T$ matrix $X$, NMF seeks a $N \times K$ matrix $W \geq 0$ (where the inequality is to be interpreted elementwise) and a $K \times T$ matrix $H \geq 0$ such that $X \approx W H$ and $K \ll \min(N, T)$.
NMF has been successfully applied to a number of practical problems; these include hyperspectral unmixing, text mining, audio source separation, and image processing;
see~\cite{cichocki2009nonnegative, gillis2014why, fu2019nonnegative} and the references therein.
\vtwominor{
Let us discuss three important limitations of NMF, each of which we will address in this paper.
}

\vtwominor{A first limitation} of NMF is that it fails to capture local correlations in the data.
For example, in imaging applications a column of $X$ may represent a pixel, and adjacent columns will often correspond to pixels adjacent to one another in the image.
Neighboring pixels tend to be quite similar, especially in low contrast images.
In audio or neuroscience datasets, each column is a certain instant in time, and therefore neighboring columns are often highly correlated.
To capture these local correlations, Smaragdis~\cite{smaragdis2004non} proposed a convolutive variant of NMF, known as convolutive NMF (CNMF).
CNMF attempts to find $L$ matrices $W_1, \hdots, W_L$ of size $N \times K$ and a matrix $H$ of size $K \times T$ such that $X = \sum_{\ell=1}^L W_\ell H S_{\ell-1}$, where $S_{\tau}$ is a square matrix that shifts the columns of $H$ by $\tau$ places to the right, zero-padding the leftmost $\tau$ columns.
Explicitly, $S_{\tau}$ has ones on its $\tau$th upper diagonal and zeros elsewhere.
It is often convenient to instead define an $L \times N \times K$ tensor $W$ and denote $W_{\ell ::} = W_\ell$ the resulting $N \times K$ matrix when the first index is fixed to $\ell$.
Then another, perhaps more intuitive, definition of CNMF is to equivalently write $X = \sum_{k=1}^K W_{::k}^T * h_k$,
    where $*$ is the 2D-convolution operator defined as
    \vtwominor{
    $(A * b)_{ij} = \sum_{\tau = 1}^L A_{i \tau} b_{j - \tau} \in \reals^{N \times T}$, with
        $A \in \reals^{N \times L}$ and
        $b \in \reals^{1 \times T}$,}
    $W_{::k} \in \reals^{L \times N}$ is the resulting matrix when the third index is fixed to $k$,
    and $h_k \in \reals^{1 \times T}$ is the $k$th row of $H$.
Thus, CNMF is a factorization of $X$ into a sum of 2D\vtwominor{-}convolutions.
This characterization is visually apparent in Figure~\ref{fig: conv-graphic}.
\begin{figure}[th!]
   \centering
   \begin{subfigure}[b]{0.85\linewidth}
       \centering
       \includegraphics[width=\linewidth]{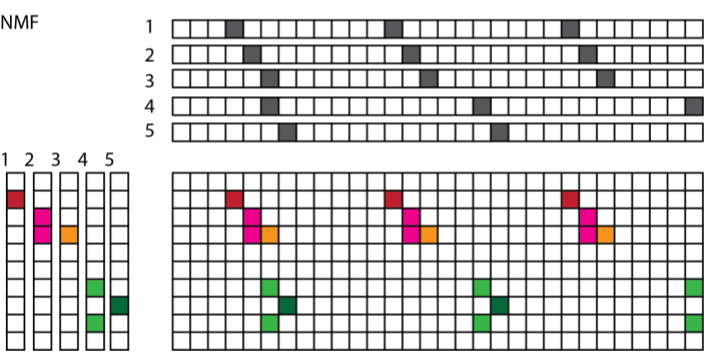}
   \end{subfigure}

   \vspace{0.75cm}

   \begin{subfigure}[b]{0.85\linewidth}
       \centering
       \includegraphics[width=\linewidth]{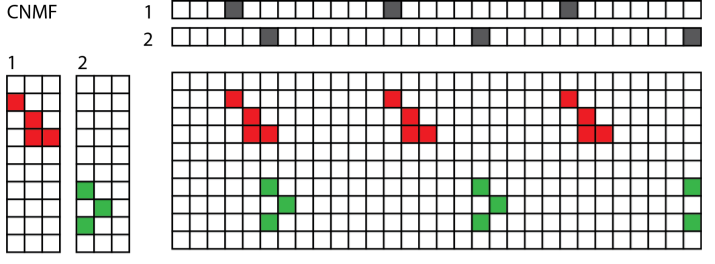}
   \end{subfigure}
   \caption{
       A visual demonstration of how NMF (top) and CNMF (bottom) attempt to reconstruct a matrix.
       NMF estimates the matrix as a sum of five outer products ($K=5$ in the NMF model), whereas CNMF reconstructs the same matrix as a sum of two convolutions ($K=2$ and $L=3$ in the CNMF model).
   }
   \label{fig: conv-graphic}
\end{figure}
Since its conception, CNMF has been found particularly useful for audio source separation~\cite{smaragdis2004non, schmidt2006nonnegative, zhou2014unsupervised} and neural sequence identification~\cite{mackevicius2019unsupervised}.
In general, CNMF will provide a more concise and interpretable factorization than NMF whenever the columns of $X$ exhibit local correlations, which is often the case when the columns represent points in time or space.


\vtwominor{A second limitation} of NMF is that
solving NMF problems is challenging in practice.
In fact, it is in general NP-hard~\cite{vavasis2009complexity}, \vtwominor{leading} researchers to rely on a number of heuristic algorithms; see, e.g.,~\cite{gillis2014why} and the references therein.
As CNMF is a generalization of NMF (the model is exactly NMF when $L=1$), it inherits problems related to computational intractability.
The first algorithm proposed was an \vtwominor{averaged} multiplicative update rule~\cite{smaragdis2004non}, which treats each pair $(W_\ell, HS_{\ell-1})$ as an NMF and updates them using an NMF multiplicative update~\cite{lee2001algorithms}, averaging the updates for $H$.
Since then, numerous other algorithms have been proposed to fit CNMF, including multiplicative updates~\cite{schmidt2006nonnegative, wang2009multiplicative, gorlow2018multiplicative, fagot2019majorization} (which are derived using  majorization minimization), projected alternating least squares~\cite{phan2012connection}, alternating nonnegative least squares (ANLS;~\cite{cnmfhals}), and hierarchical alternating least squares (HALS;~\cite{cnmfhals}).
Many of these algorithms are generalizations of algorithms for NMF and perform well on both synthetic and experimental data.
Several of them also have convergence guarantees to stationary points.
However, all the aforementioned algorithms are heuristics, in the sense that there is no guarantee
that they reach a global minimum. 

\vtwo{
A third limitation of NMF is that, in general, the NMF $WH$ of a matrix $X$ is non-unique, that is, there are other $W', H'$ such that $WH = X = W' H'$, but $W \neq W'$ and $H \neq H'$.
We therefore cannot be certain
that the recovered factors $(W,H)$ correspond to the true sources that generated the data. The same problem arises with CNMF; see Section~\ref{sec: numerical-synth} for some examples.
In practice, most researchers use additional regularization terms to promote some structure in the sought solution such as sparsity~\cite{hoyer2004non}.
Unfortunately, identifiability results for NMF are rather scarce; see the recent survey~\cite{fu2019nonnegative}.
As far as we know, the only NMF model that is both \emph{tractable} and \emph{identifiable} is separable NMF~\cite{arora2012computing}.
When trying to identify an NMF $WH = X$,
separable NMF makes the additional assumption that the columns of $W$ are contained somewhere in $X$, that is, $W = X[:,\calC]$ for some index set $\calC$.
This assumption is called the \textit{separability assumption}.
This terminology was introduced in the paper of Donoho and Stodden~\cite{donoho2004does} where it was shown that NMF has a unique solution under the condition that $W = X[:,\calC]$ \emph{and} additional sparsity conditions on $H$.
However, this NMF model dates back from the 1990's in the hyperspectral imaging literature where the separability assumption is referred to as the pure-pixel assumption; see~\cite{ma2014signal} and the references therein.
Since the paper of Arora et al.~\cite{arora2012computing}, which proved
that separable NMF model is unique, robust to noise, and computable in polynomial time,
numerous algorithms for separable NMF have been proposed~\cite{gillis2019separable}.
Moreover, this model has been used successfully in  applications such as topic modeling~\cite{arora2013practical},
community detection~\cite{panov2017consistent},
and the analysis of time-resolved Raman spectra~\cite{luce2016using}, to cite a few.
Of course, the separability assumption is rather strong and does not hold in all applications.
Nevertheless, it has also be shown to be a powerful way to initialize more sophisticated NMF models, such as minimum-volume NMF models, which are identifiable but suffer from intractability~\cite{fu2019nonnegative}.
One of the most popular and powerful separable NMF algorithm is the successive projection algorithm (SPA)~\cite{Araujo01} which is guaranteed to recover the columns of $W$, even in the presence of noise~\cite{gillis2014fast}.
SPA is a greedy algorithm that sequentially identifies the columns of $W$ using successive orthogonal projections (see Algorithm~\ref{algo: OrConSPA} and Appendix~A for more details.).
SPA is closely related to the modified Gram-Schmidt algorithm with column pivoting, and has rediscovered many times; see the discussion in~\cite{gillis2014why}.
}

\subsection{Contribution and Outline of the Paper}


In this paper, we consider conditions under which the CNMF problem can be provably solved in polynomial time, even in the presence of noise, \vtwominor{and has a unique solution (up to permutation and scaling of the rank-one a factors).}
These conditions generalize the separability assumption \vtwominor{discussed in the previous paragraph}.
We propose an algorithm that reduces CNMF to NMF with linear constraints and takes advantage of existing separability-based methods for NMF.
This algorithm provably finds the solution to CNMF in polynomial time, both for exact problems and problems with bounded noise.
In particular, we utilize \vtwominor{SPA} to estimate the columns of $W_1, \hdots, W_\ell$, then apply estimation, clustering, and sorting techniques to uncover the convolutive structure.
We later generalize our approach to use any separable NMF algorithm and show how the choice of \vtwominor{the} separable NMF algorithm affects our assumptions, run-time, and noise tolerance.
To the best of our knowledge, this is the first algorithm to provably solve the CNMF problem in either the absence or presence of noise.
\vtwominor{
Hence, under separability-like assumptions,
our approach resolved the three important limitations of NMF mentioned in the introduction.
}

The paper is organized as follows.
In Section \ref{sec: setup}, we state the CNMF problem and formally define the convolutive separability assumptions.
In Section \ref{sec: recovery}, we propose the Locate-Estimate-Cluster-Sort Algorithm (LECS) for recovery and show that LECS provably finds the solution to the CNMF problem \vtwominor{under the convolutive separability assumptions}.
In particular, Theorem \ref{thm: exact-recovery} guarantees that LECS will find the unique optimal solution (up to permutation and scaling) in the absence of noise.
Theorem~\ref{thm: noisy-recovery-abridged} generalizes this result to problems with bounded noise.
The proof of Theorem~\ref{thm: noisy-recovery-abridged} is deferred to the Appendix.

\paragraph*{Notation}
The $j$th entry of a vector $a$ is denoted by $a_j$.
For matrices $A \in \reals^{m \times n}$, we use $A[i,:] = A_{i:}$ and $A[:, j] = A_{:j}$ to denote the $i$th row of $A$ and $j$th column of $A$, respectively.
We will also use lower case letters for columns and rows, but define them explicitly first.
The entry in the $i$th row and the $j$th column of $A$ will be denoted $A[i, j] = A_{ij}$.
A nonnegative vector $a$ or a nonnegative matrix $A$ are denoted using $a \in \reals^n_+$ and $A \in \reals^{m \times n}_+$.
Define $S_\tau$ as a sqaure $T$-by-$T$ matrix with 1 on its $\tau$th upper diagonal and zeros elsewhere.
For vectors $a \in \reals^n$, the $p$-norm for $1 \leq p < \infty$ is defined as
    $\| a \|_p = \left( \sum_{j=1}^n |a_j|^p \right)^{1/p}$.
The maximum and minimum $p$-norms of a column of $A$ are denoted by
    $\| A \|_{p, col} = \max_{1 \leq j \leq n} \| A_{:j} \|_p$ and
    $\| A \|_{-p, col} = \min_{1 \leq j \leq n} \| A_{:j} \|_p$,
respectively.
We also define $\| A \|_{p, row} = \| A^T \|_{p, col}$ and $\| A \|_{-p, row} = \| A^T \|_{-p, col}$.
The Frobenius norm is denoted as
    $\| A \|_F = \left( \sum_{i, j =1}^{m, n} A_{ij}^2 \right)^{1/2}$.
We denote $\sigmax(A)$ to be the largest singular value of $A$, and $\sigmin(A)$ to be the smallest singular value of $A$.
The condition number of a matrix $A$ induced by the 2-norm is denoted as $\cond(A) = \sigmax(A) / \sigmin(A)$.
Given a matrix $A \in \reals^{m \times n}$, the diagonal matrix $D_A \in \reals^{n \times n}$ is defined by
    $(D_A)_{ii} =
    \begin{cases}
        0 & \text{ if } \| A_{:i} \|_1 = 0 \\
        \| A_{:i} \|_1^{-1} & \text{ otherwise}
    \end{cases}$.


\section{Problem Setup and the Separability Conditions}
\label{sec: setup}

Suppose there is a matrix $X \in \reals^{N \times T}_+$ generated as $X = \convsum$, where $W_\ell \in \reals^{N \times K}_+$ for $\ell=1,2,\dots,L$,
and $H \in \reals^{K \times T}_+$.
Consider $\tilde X = X + E$, where $E$ is some matrix of noise.
Given $\tilde X, K, L$, the \textit{convolutive NMF problem} (or CNMF problem) is to approximately recover $W_\ell, H$ (up to scaling or permutation).
When $E = 0$ and hence $\tilde X = X$, this problem amounts to finding an exact convolutive NMF.

\subsection{Reformulation as Constrained NMF}

Our approach will take advantage of existing literature on NMF.
In particular, we look to separable NMF algorithms and attempt to extend them to the convolutive case.
In this vein, it is useful to reformulate the convolutive NMF model as an NMF model with linear constraints of $H$.

Consider the sum $\convsum$.
Define $H_\ell = H S_{\ell-1}$ for $\ell = 1,2,\dots,L$.
Now define
\begin{align}
	\label{eq: block convsum}
	V = 	\begin{bmatrix}
				W_1 & \hdots & W_L
			\end{bmatrix} 
	\quad \text{ and } \quad
	G = 	\begin{bmatrix}
				H_1 \\ \vdots \\ H_{L}
			\end{bmatrix}. 
\end{align}
It follows that $V G = \convsum = X$.
Now we may think of the exact CNMF problem as a restriction of the exact NMF problem;
Given $X, K, L$, our goal is to find $V$ and $G$ such that $V G = X$ and $G$ is defined as in (\ref{eq: block convsum}), that is, each block $H_\ell$ is a shifted version of the first block $H_1$.
The above formulation further elucidates the relationship between NMF and CNMF.
When $L=1$, CNMF reduces to NMF.
On the other hand, any CNMF (given by $X, W_1, \hdots, W_\ell, H$) is also an NMF (given by $X, V, G$).
Therefore, \vtwominor{when the $V, G$ found from NMF have the form in (\ref{eq: block convsum}), a CNMF has also been identified.}
Given the reformulation of (\ref{eq: block convsum}), our approach is to utilize separable NMF algorithms to locate the columns of $V$ within $\tilde X$ and then estimate $G$.
Once these matrices have been identified up to permutation and scaling, we use clustering and sorting methods to identify $H$ and $W_1, \hdots, W_L$.
We name this approach the Locate-Estimate-Cluster-Sort Algorithm (LECS); see Algorithm~\ref{alg: LECS} which is described in details in Section~\ref{sec: recovery}.

\subsection{Convolutive Separability}

To introduce a notion of separability for the convolutive NMF model, we need \vtwominor{to address two relevant details that are not present in the standard NMF model. First, to handle time dependencies, we need} some notion of distance up to \vtwominor{a shift}.
The following definition formalizes this notion mathematically.


\begin{definition}[$L$-shift similarity]
    Consider two vectors $g_i, g_j \in \reals^T$.
    The $L$-shift cosine similarity $\cos_L$ between $g_i$ and $g_j$ is denoted $\cos_L(g_i, g_j)$  and given by
        \begin{align*}
        \max_{\ell=0, \hdots, L-1}
        \max\left(
        \cos\left( S_\ell^T g_i, g_j \right), \
        \cos\left( g_i, S_\ell^T g_j \right)
        \right),
        \end{align*}
    where $\cos(x, y) = (x^T y) / (\|x\|_2 \|y\|_2)$. This is exactly the \vtwominor{cosine of the minimum} angle between $g_i, g_j$ over all shifts of length $\ell < L$.
\end{definition}
The utility of the $L$-shift similarity comes from the fact that $\cos_L(x, y) = 1$ is equivalent to the statement that either $x = \alpha S_\ell^T y$ or $y = \alpha S_\ell^T x$ for some $\ell \in \{ 0, \hdots, L-1 \}$ and for some scalar $\alpha > 0$.

\vtwo{
The second detail is minor but essential.
Separable NMF algorithms like SPA require that each column of $X$ lie in the convex hull of $V$, that is, $\sum_i G[i, j] \leq 1$ for all $j$ (see Appendix~A for more details).
Nevertheless, when this does not hold we can still scale the columns of $V$ and $X$ to sum to one, and thus apply SPA.
This is not the case in the convolutive NMF model, since the magnitude of the sequence $W_{::k}$ may change over time (for example, a crescendo in an audio sequence).
Thus, we will have to allow arbitrary positive scaling of $V$ by some positive diagonal matrix $A$, and apply SPA to $VA$.
}
With these details in mind, we can now formulate a concept of separability for the CNMF problem.
We first provide a definition, then explain each condition in more details.

\begin{definition}[Convolutive separable]
    \label{def: conv-sep}
    Given a CNMF problem with inputs $\tilde X = X+E$ where $X = \convsum$, $L$ and $K$,
    we say the problem is \textit{convolutive separable} with respect to $\delta, \epsilon, A$, where $\delta > 0$, $\epsilon \geq 0$, and $A \in \reals^{KL \times KL}$ is a diagonal matrix with strictly positive diagonal entries, if the following conditions are satisfied:

	\noindent (A) \separable\
		For each $k, \ell$, the vector $W_k[:, \ell]$ appears as a scaled column of $X$.
		Explicitly, there are several equivalent ways to express this condition.

	    \noindent $\bullet$ The matrix $V$ satisfies $V A = X[:, \calC]$ for some set $\calC$.

	    \noindent $\bullet$ The matrix $G$ satisfies $G = [A\ M] \Pi$ for some $M \in \reals^{KL \times T-KL}_+$ and permutation matrix $\Pi \in \{0,1\}^{T \times T}$.

	\noindent (B) \distbound\

	        \noindent $\bullet$ For any two rows $h_i, h_j$ of $H$, the $2L$-shift similarity between them is $\cos_{2L}(h_i, h_j) \leq 1 - \delta$.

	        \noindent $\bullet$ For any row $h_i$ of $H$ and any $\ell, \tau = 1, \hdots, L-1$, we have $\cos(S_\tau^T h_i, S_\ell^T h_i) \leq 1 - \delta$ when $\tau \neq \ell$.

	\noindent (C) \fullrank\
	    The matrix $V$ is full rank, that is, $V$ has rank $KL$.

	\noindent (D)
	    \noisebound\
	    The noise matrix $E$ satisfies $\| E \|_{1, col} \leq \epsilon$.
\end{definition}

\vtwo{
Definition~\ref{def: conv-sep} guarantees a unique solution in the following sense.

\begin{theorem}[Unique Solution]
Suppose $X = \convsum$ satisfies (A), (B), and (C). Then any other CNMF given by $W_1', \hdots, W_L', H'$ that satisfies (A), (B), (C), and $X = \sum_{\ell=1}^L W_\ell' H' S_{\ell-1}$ differs at most by permutation and scaling, in the sense that $V = V' \Lambda \Pi$ for a positive diagonal matrix $\Lambda \in \reals^{T \times T}$ and a permutation matrix $\Pi \in \reals^{T \times T}$.
\end{theorem}

The proof of this result follows directly from Theorem~\ref{thm: exact-recovery}, since our proposed algorithm  (namely, LECS described in Algorithm~\ref{alg: LECS})
is deterministic and guaranteed to recover the underlying CNMF up to permutation and scaling. The same approach entails the uniqueness of a solution to a noisy CNMF problem $\tilde X = \convsum + E$ among solutions satisfying the conditions of Theorem~\ref{thm: noisy-recovery-abridged}.
}

The first condition \separable\ is inherited directly from the separable NMF literature and used only to guarantee that a separable NMF algorithm can identify the columns of $V$.
One notable exception is that in separable NMF problems, one usually requires that $V = X[:, \calC]$.
In this paper, we \vtwominor{relax} this condition to $V A = X[:, \calC]$ for some diagonal matrix $A$ with strictly positive diagonal entries.
This weaker constraint allows the columns of $V$ to be arbitrarily scaled in the matrix $X$.
However, this more general condition makes the problem no harder, since we can use a separable NMF algorithm to identify the columns of $VA$ and use this scaled matrix throughout the rest of our procedure.
In Section \ref{sec: generalize}, we show that \separable\ can actually be weakened to the more general \emph{sufficiently scattered} condition.

The second condition, \distbound\, tells us that each row of $H$ must differ from all other rows and their shifted variants.
Moreover, it tells us the rows must be \vtwominor{distinct from its own shifted variants.}
Since we are generally unconcerned with the scale of a row of $H$, the cosine of the angle between two vectors serves as a useful scale-invariant measurement of similarity
    (if the vectors have mean zero, it is exactly the correlation between the two vectors).
\distbound\ will be an important condition after we have a permuted estimate of $G$, since we will need to rearrange the rows to obtain a matrix that satisfies the constraints of (\ref{eq: block convsum}).
\vtwominor{
In practice, this is a very reasonable assumption; all we require is (a) that any two factors $h_i, h_j$ must not be highly correlated (in which case we should just have one factor), and (b) that the factor $h_j$ is not periodic with a period smaller than $L$ (in which case we can choose a smaller $L$).
}

In our approach, \fullrank\ guarantees that SPA terminates successfully and locates the columns of~$V$.
We also use it to bound the error when estimating $G$ via nonnegative least squares.
However, other algorithms like the Successive Nonnegative Projection Algorithm (SNPA) from~\cite{gillis2014successive} do not require \fullrank.
Similarly,~\cite{lotstedt1983perturbation} defines least squares robustness results in the rank deficient case.
Although we assume \fullrank\ in our results to simplify our analysis, it can be weakened to the condition that no column of $V$ is contained in the convex cone generated by the other columns of $V$.
We elaborate on this in Section \ref{sec: generalize}.

\noisebound\ bounds the 1-norm of any column of the noise matrix $E$.
When $E=0$, this is clearly satisfied.
Otherwise, we will later see that if $\epsilon$ is sufficiently small, we can still recover noisy estimates of $W_1, \hdots, W_\ell, H$.
Note that this bound differs slightly from the noise bounds in many separable NMF algorithms.
In particular, most separable NMF algorithms require bounds on the 2-norm of any column of the noise matrix $E$, whereas we bound the 1-norm.
Since $\| v \|_2 \leq \| v \|_1$ for any vector $v \in \reals^n$, we make a stronger assumption on the noise.
This stronger condition is intimately related to the following.
In separable NMF, we require $\| H \|_{1, col} \leq 1$.
In practical applications where this might not hold, we scale the columns of $X$ so that this assumption is satisfied and show that the new problem is equivalent; see for example~\cite{gillis2014fast}.
This scaling trick does not apply to CNMF however, since the matrix $H$ is difficult to appropriately scale in the expression $\convsum$.
In Appendix A, we explain this scaling trick and its limitations in more details.

\vtwo{
\emph{Is convolutive separability reasonable in practice?}
In challenging real-world scenarios, the convolutive separability condition will most likely be violated, because several sources will overlap in all time windows.
Also, the noise level will be typically higher than what our bounds allow (see Theorem~\ref{thm: noisy-recovery-abridged}).
These two issues are analogous to those encountered by separable NMF.
Nevertheless, as explained in the introduction,
separable NMF has been used successful in many challenging real-world problems, either directly or as an initialization strategy.
Accordingly, we believe many real-world signals can be approximated well under the convolutive separability assumption because, although several sources are active  in all time windows, some time windows will contain a single source that is dominant.
Moreover, as it will be shown in the experiments, our approach leads to efficient initialization strategies in challenging scenarios.
Additionally, our approach can be generalized in a number of different ways (see Section~\ref{sec: generalize}).
For example, the separability assumption can be relaxed to the sufficiently scattered condition by using minimum-volume NMF instead of separable NMF.
}

\section{Algorithm Description and Recovery Guarantee}
\label{sec: recovery}

In this section, we first describe in details our proposed algorithm (LECS, Algorithm~\ref{alg: LECS}) in Section~\ref{LECSdecr}, and then prove its correctness \vtwominor{under the convolutive separability assumptions} in the noiseless case (Section~\ref{guaranteeExact}) and in the presence of noise in (Section~\ref{guaranteeNoise}).

\begin{algorithm}[h]
\caption{Locate-Estimate-Cluster-Sort Algorithm (LECS) for the CNMF
Problem}
\label{alg: LECS}
\begin{algorithmic}[1]
    \Require An $N \times T$ matrix $\tilde X$, dimensions $L, K, R=KL$, and a parameter $t$.
    \Ensure $\tilde H, \tilde W_1, \hdots, \tilde W_L$.
    \State $\tilde V \gets$ \verb-OrConSPA-($\tilde X,\ R,\ t$).
        \label{step: SPA}
    \State $\tilde G \gets \argmin_{G \geq 0} \| \tilde V G - \tilde X \|$.
        \label{step: NNLS}
    \State $\calC_1, \hdots, \calC_k \gets$ \verb-ShiftCluster-($\tilde G,\ K,\ L$). \label{step: cluster}
    \State $\pi_k \gets$ \verb-ShiftSort-($\tilde G,\ L,\ \calC_k)$ for all $k = 1, \hdots, K$. \label{step: sort}
    \State $\tilde W_\ell[:, k] \gets V[:, \pi_k(\ell)]$ for all $k = 1, \hdots, K$ and $\ell = 1, \hdots, L$.
    \State \vtwominor{Let $a=\min(L, T-j)$}, and define $\tilde H$ to be a $K \times T$ matrix given by
        \begin{align*}
            H[k, j] &= \frac{1}{\vtwominor{a}} \sum_{\ell=1}^{\vtwominor{a}} G[\pi(\ell), j+\ell-1].
        \end{align*}
    \State \Return $\tilde H, \tilde W_1, \hdots, \tilde W_L$
\end{algorithmic}
\end{algorithm}

\subsection{Description of the LECS Algorithm} \label{LECSdecr}

\vtwo{
At a high level, our approach will be as follows: we will use the successive projection algorithm (SPA,~\cite{gillis2014fast}) to identify the column indices $\calC$ of $\tilde X$ that correspond to scaled columns of $V$, then leverage this noisy version of $V$ to obtain some estimate of $G$. Finally, we will cluster and sort the rows of $G$ to find the rows of $H$ up to some error.
}
The proposed Locate-Estimate-Cluster-Sort (LECS) algorithm is broken down into four main steps.

    \noindent $\bullet$ \textbf{Locate}. First, LECS `locates' the column \vtwominor{indices $\calC$ of $\tilde X$ that correspond to} $V' = V A D_{VA}$;
    the columns of $VA$ are present in $X$, and accordingly their noisy variants are present in $\tilde X$, so locating the column indices $\calC$ such that $VA = X[:, \calC]$ \vtwominor{will identify} a scaled, permuted variant of $V$.
    The algorithm used a modified version of SPA~\cite{gillis2014fast}, which we call Oracle Conic SPA (OrConSPA); see Algorithm~\ref{algo: OrConSPA} and Appendix A for more details.
    \vtwo{
    This method first removes the columns of $\tilde X$ that have a low signal to noise ratio using a hyperparameter $t$. Then, the algorithm rescales the columns of $\tilde X$ in order to approximately project the columns of $\tilde X$ from the conic hull of the columns of $VA$ onto the convex hull of the columns of $VA$ (the projection is approximate because of the noise). Finally, the algorithm applies SPA to recover the index set $\calC$.
    The name Oracle Conic SPA comes from the `oracle' used to select the hyperparamter $t$ and from the fact that the algorithm allows $X$ to be in the conic hull of $VA$, in contrast to standard SPA which requires $X$ to be in the convex hull of $VA$.
    This method requires $\order(NTKL)$ operations~\cite{gillis2014fast}, applying SPA being the most expensive step.
    }

    \noindent $\bullet$ \textbf{Estimate.}
    In this step, LECS estimates the rows of $G' = D_{VA}^{-1} A^{-1} G$ using nonnegative least squares (NNLS).
    Estimating $G'$ is essential because the convolutive structure is contained in its rows, which are shifted variants of the rows of $H$.
    Solving a convex NNLS problem up to any given precision can be performed in polynomial time using an interior point method (IPM).
    However, IPMs are second-order methods and hence are computationally demanding.
    We therefore instead use the block pivot method from~\cite{kim2011nnlspivot} which requires one least squares solve per iteration.
    Each least squares solve takes $\order(NK^2L^2 + NTKL)$ operations, and the algorithm almost always converges after a few iterations.

    \noindent $\bullet$ \textbf{Cluster.}
    LECS then clusters the rows of $G'$ into $K$ groups $\calC_1, \hdots, \calC_K$ according to which row of $H$ they are shifted variants of.
    One cluster should contain the $L$ shifted variants of one row of $H$.
    The algorithm achieves this by computing the $L$-shift similarity between every pair of rows in $G'$, then greedily constructing clusters by adding the available row with the highest average similarity to the rows in the cluster. This simple greedy procedure requires $\order(TK^2L^3)$ operations; see Algorithm~\ref{algo: cluster} and Lemma~\ref{lemma: clustering}.

    \noindent $\bullet$ \textbf{Sort.}
    Finally, within each cluster $C_k$, LECS sorts the rows of $G'$ based on their shifted similarity to the other rows in the cluster.
    The algorithm uses a comparison-based sorting algorithm with a comparison operator $\leq$ defined by shifted angle scores.
    In particular, for two indices $i,j$, we have $i \leq j$ if $G[j, :]$ can be better expressed as a shifted copy of $G[i, :]$ than the other way around, measured via the cosine between the two vectors.
    This comparison operator is guaranteed to produce a strict, consistent ordering when the criteria of Theorem~\ref{thm: noisy-recovery-abridged} are satisfied.
    However, note that, when these conditions fail to hold there is no guarantee that the comparison operator will produce a coherent ordering.
    Once this comparison operator has been defined, we can use any comparison-based sorting algorithm (such as merge sort or quick sort) to order the cluster $C_k$; in our implementation, we use a simple selection sort. When the operator is not consistent (that is, $i \leq j$, $j \leq k$ but $i \nleq k$ for some indices $i,j,k$), we select at each step the index which is less than or equal to the most other vectors (breaking ties arbitrarily).
    This sorting procedure requires $\order(TL^3)$ operations; see Algorithm~\ref{algo: sorting} and Lemma~\ref{lemma: sorting}.

Finally, using the clustering and sorting, the algorithm reconstructs $W_1, \hdots, W_\ell$ rearranging the columns of~$V$.
Each row of $H$ is constructed  by taking the de-shifted average of the rows of $G$ in a particular cluster.
The total computational cost of Algorithm~\ref{alg: LECS} is $\order(NTKL+TK^2L^3)$ operations plus the time to solve the nonnegative least squares problem.




\begin{algorithm}[H]
\vtwo{
\caption{Successive Projection Algorithm (\cite{gillis2014successive}, SPA) \label{algo: SPA}}
\begin{algorithmic}
    \Require An $N \times T$ matrix $\tilde X$ and a parameter $R \in \naturals$.
    \Ensure An index set $J$ with $|J| = R$.
    \State Let $B = \tilde X$, $J = \{ \}$, $r=1$.
    \While{$B \neq 0$ and $r \leq R$}
        \State $p \gets \argmax_j \| B[:, j] \|_2$.
        \State $B \gets \left( I - \frac{B_{:p} B_{:p}^T}{\| B_{:p} \|_2^2} \right)  B$.
        \State $J \gets J \cup \{ p \}$.
        \State $r \gets r+1$.
    \EndWhile
    \State \Return $J$.
\end{algorithmic}
}
\end{algorithm}

\begin{algorithm}[H]
\caption{Oracle Conic SPA (OrConSPA) \label{algo: OrConSPA}}
\begin{algorithmic}[1]
    \Require An $N \times T$ matrix $\tilde X$, the parameter $R \in \naturals$, and a threshold value $t > 0$.
    \Ensure An index set $J$ with $|J| = R$.
    \State Define a diagonal matrix $Y$ by \\
        $Y_{jj} =
        \begin{cases}
            0 & \text{ if } \| \tilde X[:, j] \|_1 \leq t \\
            \tilde X[:, j] & \text{ otherwise}
        \end{cases}$.
    \State $\tilde X' \gets \tilde X Y$.
    \State $J \gets $ \verb_SPA_($\tilde X',\ R$).
    \State \Return $J$.
\end{algorithmic}
\end{algorithm}

\begin{algorithm}[H]
\caption{Shift Cluster Algorithm \label{algo: cluster}}
\begin{algorithmic}[1]
    \Require A $KL \times T$ matrix $\tilde G$ and parameters $K,\ L$.
    \Ensure $\calC_1, \hdots, \calC_k$
    \State $\calA = \{1, \hdots, KL\}$
    \For{$k = 1, \hdots K$}
        \State Choose $i_1 \in \calA$ arbitrarily.
        \State Find $i_2, \hdots, i_{L}$ by solving
            $\argmax_{i_2, \hdots, i_{L} \in \calA}\ \sum_{\ell=2}^{L} \cos_{L}(\tilde G[i_1, :], \tilde G[i_\ell, :])$.
        \State $\calC_k \gets \{i_1, \hdots, i_L \}$, $\calA = \calA \setminus \{i_1, \hdots, i_L \}$.
    \EndFor
    \State \Return $\calC_1, \hdots, \calC_k$.
\end{algorithmic}
\end{algorithm}

\begin{algorithm}[H]
\caption{Shift Sort Algorithm \label{algo: sorting}}
\begin{algorithmic}[1]
    \Require A $KL \times T$ matrix $\tilde G$, the parameter $L$, and a set of $L$ indices $\calC$.
    \Ensure A map $\pi : \{1, \hdots L\} \rightarrow \calC$.
    \State For each $i, j \in \calC$, compute
    \begin{align*}
        \mu_{\text{left}}(i, j)
        = \max_{\ell=0,\hdots,L-1}\ \cos( S_\ell^T \tilde G[i,:] - \tilde G[j, :] ),
        \\
        \mu_{\text{right}}(i, j)
        = \max_{\ell=0,\hdots,L-1}\ \cos( \tilde G[i,:] - S_\ell^T \tilde G[j, :] ).
    \end{align*}
    \State Define a comparison operator on two indices $i, j$ by
    \begin{align*}
        (i \leq_{L, \tilde G} j)
        =
        \begin{cases}
        \text{True} &
            \text{if }
            \mu_{\text{left}}(i, j) \geq \mu_{\text{right}}(i, j) \\
        \text{False} & \text{otherwise}
        \end{cases}.
    \end{align*}
    \State Sort $\calC$ using the comparison operator $\leq_{L, \tilde G}$. Let $\pi$ be the resulting indexed list (with indices  $\{1, \hdots, L\}$).
    \State \Return $\pi$.
\end{algorithmic}
\end{algorithm}

\subsection{Guarantee for the Exact Problem} \label{guaranteeExact}

The exact problem is when $E = 0$.
Then \noisebound\ is satisfied with $\epsilon=0$.
In this case, Algorithm~\ref{alg: LECS} provably recovers $H$ and each $W_\ell$ for any $\delta > 0$ and any $A$ such that $\min_i A_{ii} > 0$.

\begin{theorem}[Exact Recovery]
\label{thm: exact-recovery}
    Given a CNMF problem that is convolutive separable with respect to $\epsilon = 0$ and some $\delta, A$, Algorithm \ref{alg: LECS} with inputs $\tilde X = X, K, L$ recovers the unique factorization $H, W_1, \hdots, W_\ell$ satisfying (A), (B), and (C), up to a permutation and scaling, in polynomial time; more precisely $\order(NTKL + TK^2L^3)$ flops plus one NNLS solve.
\end{theorem}

\begin{proof}[Proof of Correctness]
	Because of conditions \fullrank\ and \separable, SPA is guaranteed to find the columns of $V$.
	\fullrank\ also guarantees nonnegative least squares will return the rows of $G$.
	All that remains is to show the grouping correctly recovers $W$ and $H$.

	Given any two rows $g_i, g_j$ that are both shifted versions of some row of $H$, their similarity is $\cos_L(g_i, g_j)=1$.
	Rows that are not shifted variants of one another cannot have similarity greater than or equal to $1-\delta$ by condition \distbound.
	This means the grouping given by the clustering step is necessarily correct.

	\distbound\ also entails that the ordering from the sort step is also correct.
	Since the ordering is correct, the construction of $H$ and each $W_\ell$ must also be correct up to a permutation and scaling of the rows of $H$.

	Uniqueness follows from the success of the algorithm.
	Suppose we have two factorizations $W_1, \hdots, W_L, H$ and $W_1', \hdots, W_L', H'$ such that $X = \convsum = \sum_{\ell=1}^L W_\ell' H' S_{\ell-1}$.
	Then given $X$, Algorithm \ref{alg: LECS} recovers the first factorization up to permutation and scaling.
	However, Algorithm \ref{alg: LECS} also recovers the latter factorization up to permutation and scaling.
	Hence it must be that one factorization is a permute, scaled version of the other, that is, $H = \Pi \Lambda H'$ for some permutation matrix $\Pi \in \{0,1\}^{K \times K}$ and some diagonal matrix $\Lambda \in \reals^{K \times K}$ with positive diagonal elements.
\end{proof}

\subsection{Guarantee for the Noisy Problem} \label{guaranteeNoise}

Even when $E \neq 0$, we can still recover noisy estimates of $H, W_1, \hdots, W_\ell$ given sufficiently low noise levels.

\begin{theorem}[Noisy Recovery]
\label{thm: noisy-recovery-abridged}
    Suppose a CNMF problem with inputs $\tilde X, L, K$ is convolutive separable with respect to $\delta, \epsilon, A$ and consider some parameter $t > 0$.
    Let $V' = V A D_{VA}$ and $G' = D_{VA}^{-1} A^{-1} G$ and suppose
    \begin{align}
        \epsilon + t
        &< \| V A \|_{-1, col} ,
            \label{eq: abridge-thresh}
            \\
        \frac{\epsilon}{t}
        &< C_a \left( \frac{\sigmin(V') }{ \sqrt{KL} \cond(V')^2 } \right),
            \label{eq: abridge-scale-noise}
            \\
        \epsilon
        &<
        C_b \left( \frac{\rho \delta}{\| G' \|_{2, col} \sqrt{T}} \right),
            \label{eq: abridge-noise}
    \end{align}
    where $\rho =
        C_a \left( \frac{\sigmin(V') }{ \sqrt{KL} \cond(V')^2 } \right)-
        \frac{\epsilon}{t}$
    and $C_a, C_b$ are universal constants independent of all other terms.
    Furthermore, without loss of generality assume that $\| V A \|_{2, col} > 1$ and $t > 1$ (see below).
    Then in polynomial time, more precisely $\order(NTKL + TK^2L^3)$ flops plus one NNLS solve, Algorithm \ref{alg: LECS} finds $\tilde H \in \reals^{K \times T}$ with bounded error, in the sense that there exists some permutation $P$ such that
    \begin{align*}
        \min_{1 \leq j \leq K}\ \cos\left( h_j, \tilde h_\permj \right)
        \geq
        1 - \frac{\delta}{2},
    \end{align*}
    where $h_j, \tilde h_\permj$ are the $j$th and $\permj$th rows of $H$ and $\tilde H$, respectively.
\end{theorem}

The proof of this theorem is deferred to \vtwominor{Appendix~\ref{apdx: noisy-proof}, where we prove a more general statement in Lemma \ref{lemma: noisy-recovery-full}}.
The assumptions $\|VA \|_{2, col} > 1$ and $t > 1$ are only used to simplify (\ref{eq: abridge-noise}) and make the bound easier to read.
These assumptions are also made without any loss of generality---we can always consider the equivalent problem with $\alpha V$ and $\alpha \tilde X = \alpha X + \alpha E$ for an arbitrarily large $\alpha > 0$;
in particular for $\alpha = \max(t^{-1}, \|VA\|_{2, col}^{-1})$.
This scaling does not impact the output of the algorithm (beyond scaling), and cancelling terms shows it also does not affect inequality~(\ref{eq: abridge-thresh}) and inequality~(\ref{eq: abridge-scale-noise}).
However, one should note this scaling does impact $\| G' \|_{2, col}$ since $G'$ depends on $D_{VA}^{-1}$, which changes inequality (\ref{eq: abridge-noise}).

Algorithm~\ref{alg: LECS} depends on selecting a good value for the parameter $t$; the best error bounds is obtained as $t \rightarrow \|VA\|_{-1, col} - \epsilon$.
This parameter is used to threshold columns of $\tilde X$ with a small norm, since these columns will have relatively low signal-to-noise ratios and cause SPA to fail after scaling $\tilde X$.
In practice, one can run the algorithm several times using different values of $t$ and keep the best fit according to some heuristic (for example, mean square \vtwominor{reconstruction} error).
\vtwo{
In fact, since the conditions of the Theorem \ref{thm: noisy-recovery-abridged} imply that noisy versions of the columns of $VA$ are present in $\tilde X$, we can be certain that $\|\tilde X[:, p] \|_{-1, col}$ is within $\epsilon$ of $\|VA \|_{-1, col}$ for some $p \in \{ 1, \ldots, T \}$. So we could run the algorithm $T$ times with parameter $t = \|X[:, p]\| - 2\epsilon$
for $p=1,2,\dots,T$
and guarantee that one of the runs yields the desired solution.
}

\section{Generalizations}
\label{sec: generalize}

One strength of our approach is that any of the subroutines used for the four steps (location, estimation, clustering, and sorting) can be substituted in exchange for different subroutines.
For example, we use SPA to locate the columns of $V$, but could substitute any separable NMF algorithm in its place.
In this section, we provide several examples of improvements to LECS obtained by appropriate substitutions and highlight the generality of the high level algorithm.


\paragraph{Rank-deficient identification of $V$}
We previously mentioned that \fullrank\ is required for two reasons:
    to guarantee the success of the SPA algorithm and
    to bound the error when estimating $G$ via nonnegative least squares.
However, if we replace SPA with another separable NMF algorithm that does not require $V$ to have full column rank, then the first reason is no longer necessary.
SNPA in particular essentially requires the weaker condition that no column of $V$ can be written as a conic combination of the other columns of $V$.
By using this method in place of SPA, we can complete the location step with weaker condition than \fullrank.
This condition is also not necessary for estimating $G$, which we discuss later in the section.

\paragraph{Sufficiently Scattered Condition}
Since \separable\ is only used in identifying the columns of $V$, we could also use the \emph{sufficiently scattered}~\cite{huang2013non, fu2019nonnegative}, which is more general and allows one to determine $V$ by identifying the minimum-volume NMF of $X$.
Unfortunately, this more general condition suffers from two drawbacks:
    first, it has yet to be proven robust to noise, and, second, there are currently no algorithms \vtwominor{that are guaranteed to find} the minimum-volume NMF in polynomial time~\cite{fu2019nonnegative} (this problem is NP-hard in general, but it remains unknown whether this is still true under the sufficiently scattered condition).


\paragraph{Rank-deficient Estimation of $G$}
Under some assumptions, the matrix $V$ need not be full rank to properly estimate $G$.
In particular,~\cite{lotstedt1983perturbation} gives perturbation bounds in the rank-deficient case in a more general setting;
these bounds apply to our problem when the rank of $\tilde V$ is the same as the rank of $V$.

\paragraph{Robust Least Squares}
When we estimate $G$, we solve the problem $\min_{G \geq 0} \| \tilde V G - \tilde X \|_F^2$ in an attempt to find some $\tilde G$ with rows similar to those of $G$.
If we know the noise level $\epsilon$ in advance or have some estimate of it, the optimal estimate is given by the robust optimization problem:
    $\min_{\tilde G \geq 0} \max_{(V, G, X) \in \calC}
        \| \tilde G - G \|_{2, row}$, where
        $\calC = \{(V, G, X) :
        \| V - \tilde V \|_{2, col} \leq \epsilon_V,
        \| X - \tilde X \|_{1, col} \leq \epsilon_X,
        V, G, X \text{ satisfy (\ref{eq: block convsum}) and (A)-(D)} \}$,
where $\epsilon_V, \epsilon_X$ depend on $\epsilon$.
In general, a number of least squares-like procedures may be used in place of standard nonnegative least squares in order to minimize the worst case error $\|\tilde G - G \|_{2, row}$.


\paragraph{Generalized Similarity Metric}
The $L$-shift similarity measure utilizes the cosine of the angle between two vectors because it is scale invariant, allowing us to consider $a$ and $\alpha a$ equivalent, where $a \in \reals^n$ and $\alpha \in \reals, \alpha > 0$.
The choice of similarity measure heavily influences the clustering and sorting steps, as seen in the proofs of Lemma \ref{lemma: clustering} and Lemma \ref{lemma: sorting}.
Therefore, a wise adjustment of this measure could lead to significant improvements in the final two steps of LECS.
This might include a simple post-processing step, such as exponentiation, or an entirely different similarity measure.
See~\cite{ng2002spectral, favati2019similarity} for a more comprehensive review.

\paragraph{Spectral Clustering}
Our greedy clustering algorithm is simple but suboptimal, in the sense that it does not take advantage of the global cluster structure.
One more advanced approach is to use a spectral clustering method~\cite{ng2002spectral}.
Let the matrix $M$ be defined as
\begin{align*}
    M_{ij} =
    \begin{cases}
        1 & \text{if } \exists k \in \{1, \hdots, K\} \text{ such that } i, j \in \calC_k,  \\
        0 & \text{otherwise.}
    \end{cases}
\end{align*}
Let $v_1, \hdots, v_k$ be the eigenvectors corresponding to the $K$ largest eigenvalues of $M$.
Orient each eigenvector so that $|\max_j (v_i)_j| \geq |\min_j (v_i)_j|$ (that is, ensure the entry with the largest magnitude is positive).
Then the indices of the largest $L$ entries in $v_i$ form the representatives for cluster $i$.
The challenge is to construct a similarity matrix $\tilde M$ close to $M$ using only the observed data.
One option is to construct $\tilde M$ using the $L$-shift similarity, then set the largest $KL^2$ entries to 1 and the remaining entries to 0.
Under the conditions of Theorem~\ref{thm: noisy-recovery-abridged}, $\tilde M = M$ and therefore recovery is guaranteed.
Results from matrix perturbation theory, such as the Davis-Kahan bound~\cite{davis1970rotation}, suggest there is a strong theoretical justification for using a spectral clustering method, given the right construction of $\tilde M$.

\section{Numerical Experiments}

We first test the LECS algorithm on synthetic data and verify it correctly finds the ground truth in low noise settings.
Then, we demonstrate that LECS finds reasonable results on the spectrogram of a songbird and can be used as an effective initialization for other algorithms like multiplicative updates~\cite{schmidt2006nonnegative} and alternating nonnegative least squares (ANLS)~\cite{cnmfhals}.

All code is written in Julia~\cite{julia} and available on GitHub at \url{github.com/degleris1/CMF.jl}.
In particular, the code used at the time of publication is available in release v0.1.
We use the code from \url{github.com/ahwillia/NonNegLeastSquares.jl} to solve NNLS problems via the pivot method.
The figures for Sections \ref{sec: numerical-synth} and \ref{sec: numerical-song} are produced by the Jupyter notebooks \verb|figures/sep_synth.ipynb| and \verb|figures/sep_song.ipynb|,
and the LECS algorithm itself is available file \verb|src/algs/separable.jl|.
Note that this repository is available to other researchers as a tool for fitting CNMF models, as well as rapidly developing and testing new CNMF algorithms.

\subsection{Unstructured Synthetic Data}
\label{sec: numerical-synth}

\begin{figure}[t]
    \centering
    \includegraphics[width=\linewidth]{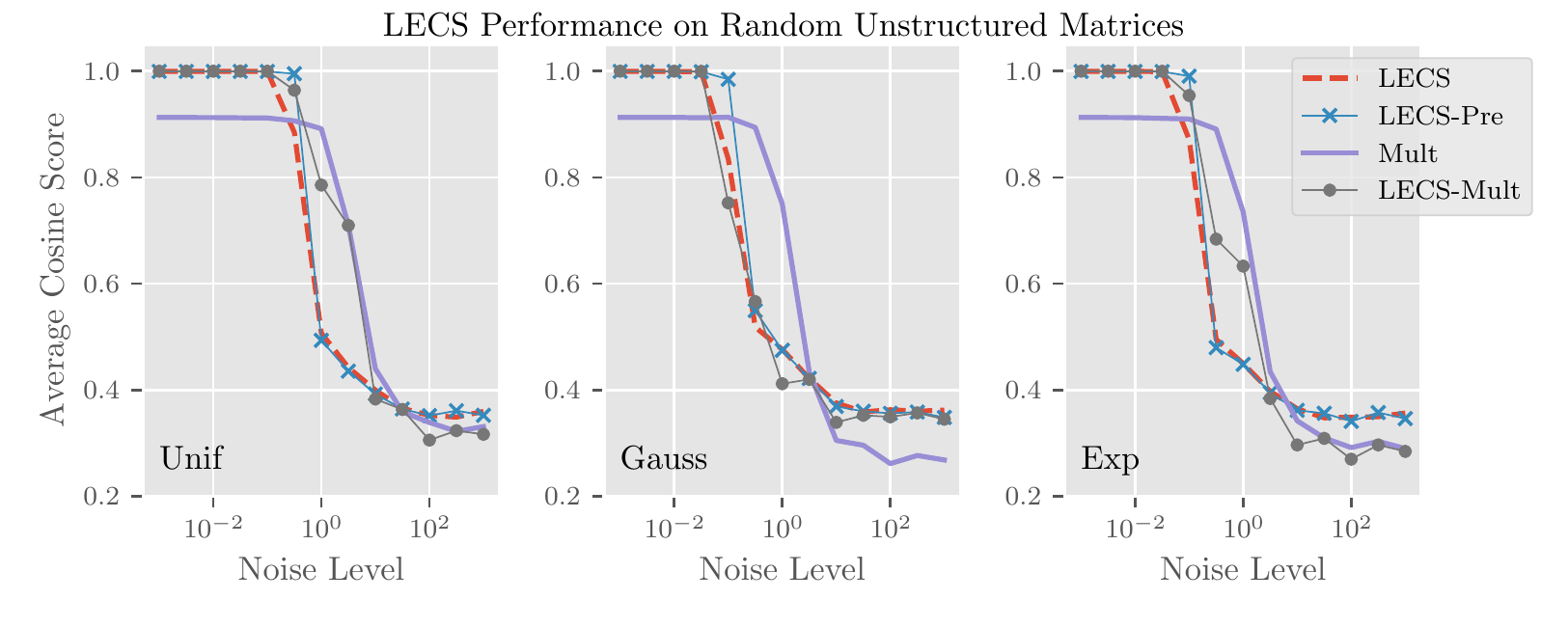}
    \caption{Performance on unstructured synthetic data.}
    \label{fig: synth-performance}
\end{figure}

In this experiment, we consider synthetic data that is separable by construction, but lacks structure in the sense that matrices are generated randomly without constraints on important parameters like condition number.
Since angle determines the uniqueness of a vector up to scaling and plays an important role in our theoretical analysis, we use it to measure performance.
Specifically, after the algorithm estimates $\tilde H$, the score is computed as the average cosine of the angles between the rows of $\tilde H$ and the rows of $H$ (minimized over all row permutations).
This is formally given by
\begin{align*}
    \text{score}(H, \tilde H)
    = \min_{P \in \calP}\quad
    \frac{1}{K} \sum_{k=1}^K \cos(H[k, :], \tilde H[P(k), :]),
\end{align*}
where $\calP$ is the set of all permutations.
We test four different algorithms:

    \noindent $\bullet$ \textbf{LECS.}
        The LECS algorithm
        as presented in Algorithm~\ref{alg: LECS}.

    \noindent $\bullet$ \textbf{LECS-Pre.}
        The LECS algorithm with SPA replaced by heuristically preconditioned SPA from~\cite{gillis2015semidefinite}.
        Heuristically preconditioned SPA computes the SVD $U \Sigma V^T$ of $\tilde X$, then runs SPA on $\Sigma^{-1} U^T \tilde X = V^T$. Under the assumption that the data points are evenly distributed within the convex cone generated by $V$, it is more robust than SPA~\cite{gillis2015enhancing}. This illustrates the flexibility of LECS with respect to the choice of the different building blocks; see Section~\ref{sec: generalize}.

    \noindent $\bullet$ \textbf{Mult.}
        The multiplicative updates algorithm for minimizing the Frobenius norm used in~\cite{schmidt2006nonnegative, mackevicius2019unsupervised}.
        This also corresponds to the $\beta$-divergence update rule from~\cite{gorlow2018multiplicative, fagot2019majorization} when $\beta = 2$.
        The algorithm is initialized randomly and scaled using the method from~\cite{ho2008thesis}. The alternating update rules runs for 5 seconds.

    \noindent $\bullet$ \textbf{LECS-Mult.}
        The same multiplicative updates algorithm as before, except using {LECS} as its initialization.

For each $\ell \in \{1, \hdots, L\}$, we let $(W_\ell)_{ij} \sim \uniform(0.5, 1.5)$ for all $i,j$, where $\uniform(a,b)$ is the uniform distribution in the interval $[a,b]$.
We let $H_{ij} \sim \uniform(0, 1) \bern(1-p)$ for all $i,j$, where $p$ is our sparsity parameter and $\bern(1-p)$ is the Bernoulli distribution with parameter $1-p$.
\vtwominor{To enforce the separability condition, we then set some entries of $H$ to zero as follows.}
\vtwominor{For each row $H[k,:]$ of $H$,} choose two random $t_k, s_k \in \{0, \hdots T-L \}$.
For all $t$ satisfying either $t_k - L \leq t \leq t_k + \lfloor L/2 \rfloor$ or $s_k - \lfloor L/2 \rfloor \leq t \leq s_k + L$, we set $H[:, t] = 0$.
Then we let $H[k, t_k], H[k, s_k] \sim \uniform(0.5, 1.5)$.
Finally, we ensure each $t_1, s_1, \hdots, t_K, s_K$ are sufficiently far apart from one another so that the same entries are never set to zero twice.
\vtwominor{This construction ensures the first half of each sequence ($W_{::k}$) is separable at one location, and the second half at another location.}
We choose ``half-sequences'' to demonstrate the full sequence need not be \vtwominor{separable at the same location}.
To generate the noise, we vary a \vtwominor{noise parameter $\beta$} and use one of three procedures: for all $i,j$

    \noindent $\bullet$ \textbf{Uniform.}
        Set $E_{ij} \sim \uniform(0, \beta)$.

    \noindent $\bullet$ \textbf{Gaussian.}
        Set $E_{ij} = \max(-X_{ij}, M_{ij})$ where $M_{ij} \sim \normal(0, \beta)$.

    \noindent $\bullet$ \textbf{Exponential.}
        Set $E_{ij} \sim \exponential(\beta)$.

We generate $n_{\text{trials}} = 10$ random matrices $X$ according the above construction with $N = 100$, $T = 250$, $K=3$, $L=5$, $p=0.75$.
Then, for each matrix, the noise level $\beta$ is varied across $n_{\text{noise}} = 13$ different values spaced logarithmically between $\beta = 10^{-3}$ and $\beta = 10^{3}$.

\vtwo{The results of these experiments are displayed in Figure~\ref{fig: synth-performance}.}
We observe that for all three noise types, LECS finds the true $H$ for small noise levels, as expected by Theorem~\ref{thm: noisy-recovery-abridged}.
Additionally, as expected, the preconditioning improves its tolerance to noise.
In comparison, Mult fail to find the true $H$ even in the lowest noise settings. \vtwo{In fact, although Mult finds a solution with low reconstruction error, it is not able to recover the ground truth factor $H$:
this is due to the non-uniqueness problem with CNMF (see Section~\ref{sec: intro}).}
\vtwominor{As a matter of fact}, Mult with LECS as an initialization keeps the true solution and outperforms random initializations.
At higher noise levels, the effects are varied.
In particular, LECS-Mult exhibits different performance depending on the type of noise, and notably performs the best (relative to random initialization) when the noise is Gaussian.
This is likely due to the fact this update rule attempts to minimize the reconstruction error using a Frobenius loss function, which is a natural objective for Gaussian noise but not uniform nor exponential noise.
A different loss function is likely necessary to identify the ground truth given these noise models;
for example, entrywise $\infty$-norm and $1$-norm losses may be more appropriate for uniform and exponential noise, respectively.

\subsection{Songbird Spectrogram}
\label{sec: numerical-song}

In practice, noise levels are often too high for the LECS algorithm to identify the ground truth, even when the separability assumptions are satisfied.
However, LECS can be used as an initialization to other algorithms for CNMF to achieve faster convergence than random initialization, and to obtain better solutions.

In this experiment, we fit a spectrogram of a singing bird from~\cite{mackevicius2019unsupervised};
the authors have kindly made this dataset publicly available at \url{github.com/FeeLab/seqNMF}.
The spectrogram matrix $\tilde X$ has 141 rows (DFT bins) and 4440 columns (timebins).
We fit the spectrogram using the LECS algorithm with $K=3,\ L=20$, and a threshold value $t=10$.
Then, we use this initialization to run Mult~\cite{schmidt2006nonnegative} and ANLS~\cite{cnmfhals} algorithms for 15 and 60 iterations, respectively. Note that ANLS, as for NMF~\cite{kim2011nnlspivot}, solves the subproblem exactly for $W$ and $H$ using an active set method, as opposed to Mult that is a gradient-like descent method.
The threshold value $t$ was chosen after sweeping over all values from 0 to 50 (approximately the maximum 1-norm of a column of $\tilde X$).

Figure \ref{fig: song-loss} displays the loss curves for Mult and ANLS using both random initialization and the output of LECS.
The loss is measured as the relative mean square error, defined as
\begin{align*}
    \text{MSE}_{\text{rel}} = \frac{\| \tilde X - \convsum \|_F}{\| \tilde X \|_F},
\end{align*}
where $W_1, \hdots, W_L, H$ are the estimated matrices.
\begin{figure}[h]
    \centering
    \includegraphics[width=0.5\linewidth]{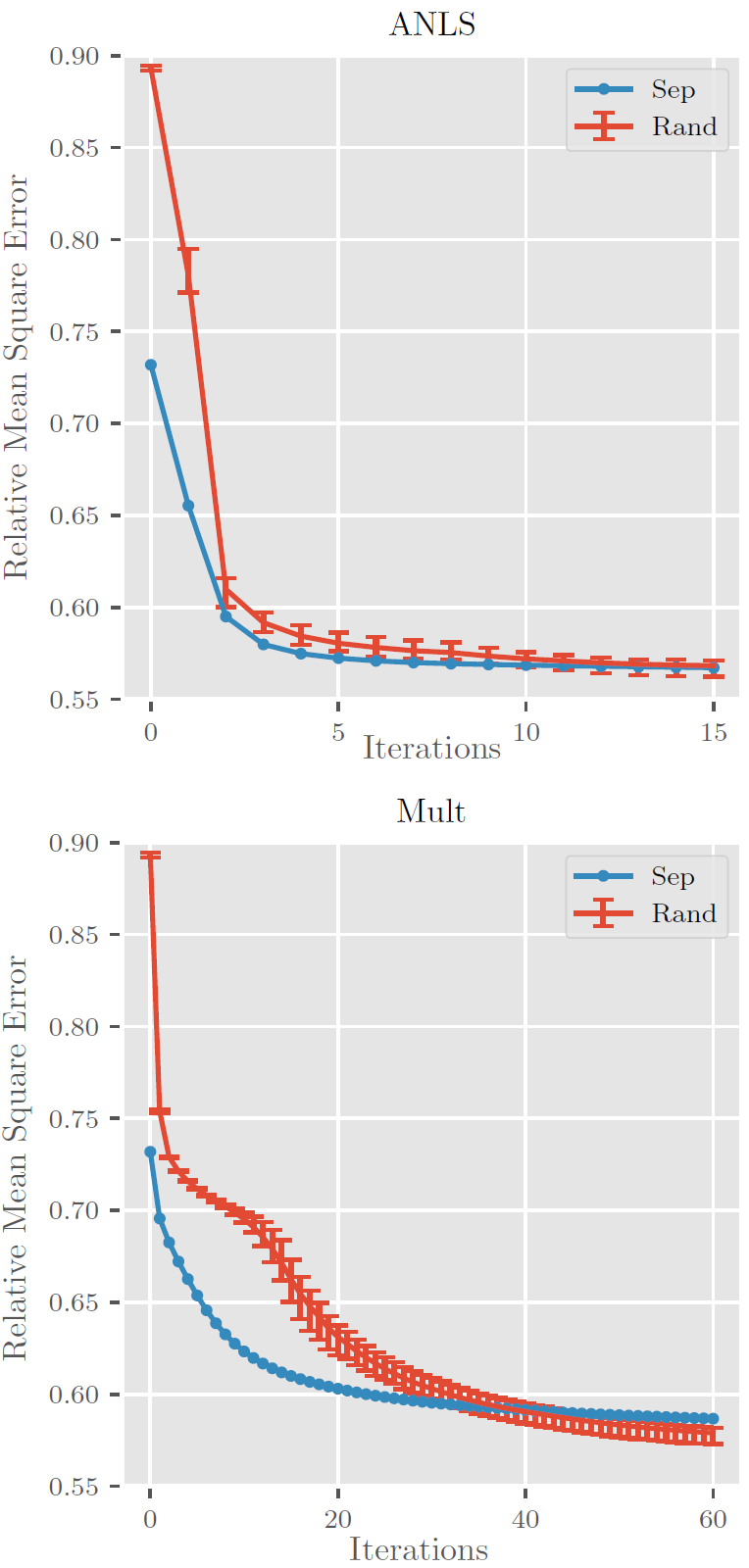}
    \caption{
    Loss plots from fitting the songbird spectrogram using both random initialization (median from 30 trials) and the output of LECS as initialization.
    Error bars represent the 10th and 90th percentiles of the 30 trials.
    }
    \label{fig: song-loss}
\end{figure}

We observe that when either Mult or ANLS are initialized using the output of LECS, they converge to a local minimum in fewer iterations.
This suggests that LECS finds a solution near some local minimum and could be used to accelerate algorithms on very large datasets. In fact, when used as an initialization for ANLS, LECS allows to obtain a high quality solution faster than random initializations. For Mult, using LECS as an initialization does not perform as well because, on average and after sufficiently many iterations, the solution generated by random initializations have lower reconstruction error. We suspect that the reason is that Mult does not deal well with input matrices with many small entries (which LECS generates) because of the zero locking phenomenon--zero entries cannot be modified by Mult while it takes many iterations for small entries to get large; see for example the discussion in~\cite{gillis2014why} and the references therein. In fact, with the LECS initialization, \vtwominor{ANLS reduces} the relative error to 56.6\% after 15 iterations, while MU reduces it to only 58.4\% after 60 iterations. Note also that ANLS performs on average better than MU with an average relative error of 56.6\% for ANLS vs.\@ 57.7\% for MU; this was already observed~\cite{cnmfhals}.

Moreover, despite a high relative mean square error, Figure~\ref{fig: song-qual} demonstrates that LECS outputs factors and a reconstruction with qualitative similarities to the ground truth.
In fact, after just a single iteration of the ANLS algorithm, the factors visibly mimic the sequential structure in the dataset.
\vtwominor{Note that, as opposed to the experiments on the synthetic data set performed in the previous section, we cannot assess the quality of the generated factors as the ground truth sources are not available.}

\afterpage{
\begin{figure}[tp]
    \centering
    \begin{subfigure}{0.8\linewidth}
        \centering
        \includegraphics[width=\linewidth]{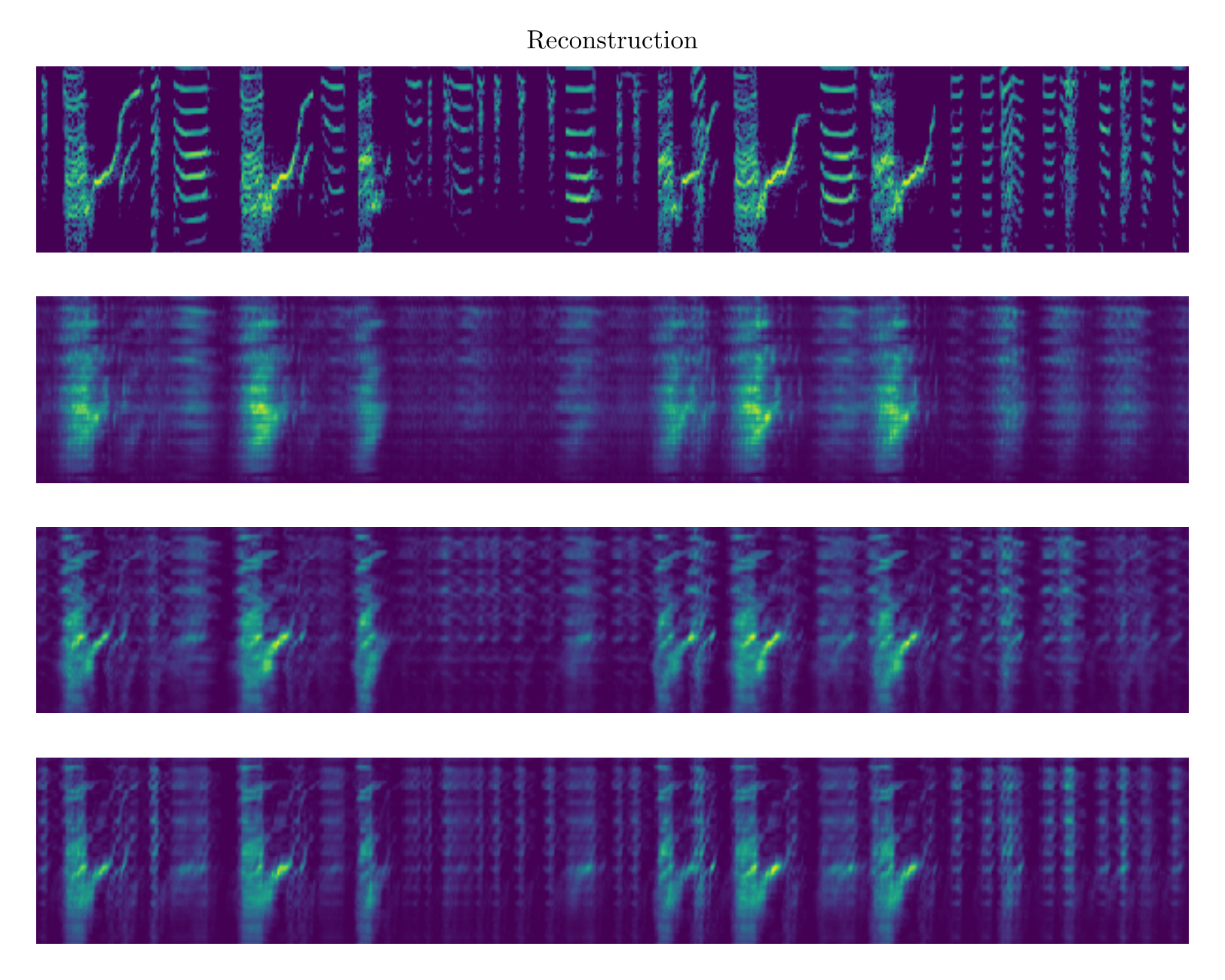}
        \caption{}
    \end{subfigure}

    \begin{subfigure}{0.8\linewidth}
        \centering
        \includegraphics[width=\linewidth]{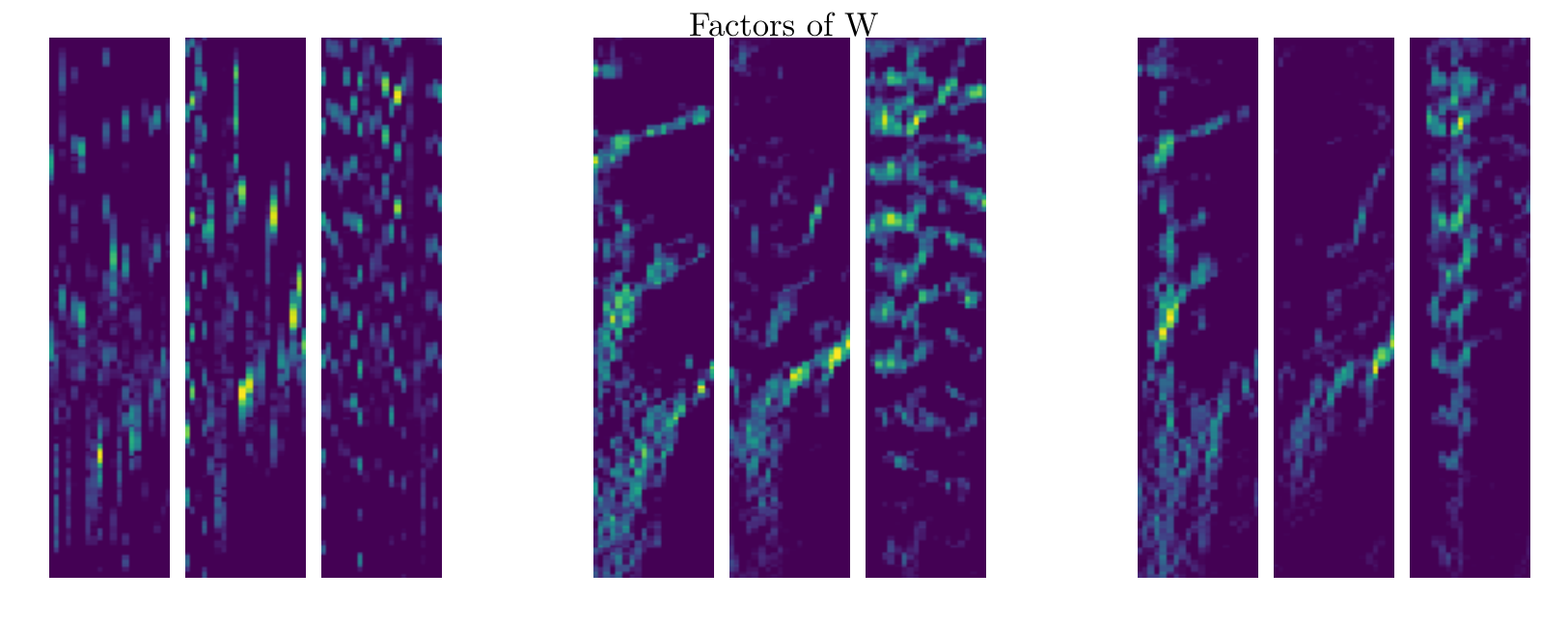}
        \caption{}
    \end{subfigure}

    \caption{
        Qualitative results on the songbird spectrogram.
        The top figure \textbf{(a)} contains 1000 timebins of the true spectrogram (top), the reconstruction produced by LECS (middle top), and the reconstruction after a single and two iterations of ANLS initialized with LECS (middle bottom and bottom). The bottom figure \textbf{(b)} shows the factors $W_{::1}, W_{::2}, W_{::3}$ produced by the LECS algorithm (left three; see Section \ref{sec: intro} for the definition of $W_{::k}$) and after a single iteration and two iterations of ANLS initialized with LECS (middle three and right three).
    }
    \label{fig: song-qual}
\end{figure}
\clearpage
}

\section{Conclusion}

In this paper we have presented a provably correct and robust algorithm for CNMF under separability conditions.
By addressing the questions of identifiability and provable recovery of the solution, this paper offers a novel theoretical perspective on the CNMF problem.
Our algorithm draws directly from methods for separable NMF and is easily generalizable---the high level algorithm does not require any particular method for each of its four steps.

Future work includes addressing two significant weakness of our approach.
First, the location step does not leverage the convolutive structure of CNMF (and could therefore possibly be made more tolerant to noise).
Second, the method cannot handle the case when some $W_{::k}$ has repeated columns, which is relevant to processing music datasets that often contain pure harmonic tones.

\appendix
\section{Appendix}

\subsection{Conic SPA Procedure}
\label{apdx: conic-spa}

\vtwo{
As mentioned in the introduction, SPA is a sequential algorithm for separable NMF that identifies the columns of $V$ among the columns of $X=VG$.
At each step, it first extracts the column of $X$ that has the largest $\ell_2$ norm, and then project all columns of $X$ onto the orthogonal complement of the extracted column; see Algorithm~\ref{algo: SPA}.
Under the assumptions that $V=X[:,\calC]$ is full column rank and the columns of $G$ have unit $\ell_1$ norm, SPA recovers the set $V$, even in the presence of noise.
In this section we generalize SPA from an algorithm for finding the vertices of a polytope (because the columns of $G$ are required to have unit $\ell_1$ norm) to an algorithm for finding the extreme rays of a cone.
In the noiseless case, it is relatively straightforward to show this generalization by an appropriate rescaling of $X$.
However, when noise is present, extra steps must be taken to ensure points that are primarily noise do not grow too large and are mistaken for a vertex.
}
In fact, in the robustness analysis of SPA, it is assumed from the scratch that the input noisy data matrix has the form $\tilde X = VG + E$, where the columns of $G$ are normalized which is key for SPA to succeed.
If $G$ is not normalized, then normalization of the input matrix is necessary and this might increase the noise drastically; for example for the columns of $X$ with very small norm. Therefore, to make SPA robust to noise in this scenario, we will need to first remove the columns of $\tilde X$ with small norm.

\begin{lemma}[SPA]\cite[Theorem 3]{gillis2014fast}
\label{lemma: spa-noisy}
    Let $\tilde X = X + E = V G + E \in \reals^{N \times T}$, where
        $V \in \reals^{N \times R}$ has rank $R$ and
        $G \in \reals^{R \times T}_+$ satisfies $\sum_i G[i, j] \leq 1$ for all $1 \leq j \leq T$.
        Additionally, suppose that $G = [I\ M] \Pi$ for some permutation matrix $\Pi$ and $M \in \reals^{R \times T-R}_+$.

    Let $\| E[:,i] \| \leq \epsilon$ for all $i$ with
        $\epsilon
        <
        \frac{C_1 \sigmin(V)}{\sqrt{R} \cond(V)^2}$,
    where $C_1$ is a global constant. 
    Let $J$ be the index set extracted by SPA, which takes time $\order(NTKL)$. Then there exists a permutation $P$ of $\{1, \hdots, R\}$ such that
    \begin{align*}
        \max_{1 \leq j \leq R}\ \| \tilde X[:,J(j)] - V[:,P(j)] \|
        =
        C_2 \epsilon \cond(V)^2,
    \end{align*}
    where $C_2$ is a global constant. 
\end{lemma}

Note that Lemma \ref{lemma: spa-noisy} makes two strong assumptions on $G$;
first, that $\sum_i G[i, j] \leq 1$ for all $1 \leq j \leq T$, and second, that $G$ contains the identity matrix as a submatrix.
Because of the additional structure of the convolutive NMF problem, it is unreasonable to have both these assumptions simultaneously. Instead, we drop the first assumption and generalize the SPA algorithm to any $G$ such that
    $G = [I\ M] \Pi \in \reals^{R \times T}_+$,
    where $\Pi$ is a permutation matrix and
    $M \in \reals^{R \times T-R}_+$.
Note that this formulation is equivalent to $G = [\Lambda\ M]\Pi$, where $\Lambda$ is some full rank diagonal matrix, since we can always scale the $j$th row of $G$ by $\lambda_j$ and scale the $j$th column of $V$ by $\lambda_j^{-1}$ without changing $VG$.
We first demonstrate conditions under which we can recover $V$ if some oracle gives us a threshold value $t$ which is greater than the noise level $\epsilon$ but smaller than $\|V\|_{-1} - \epsilon$.

\begin{lemma}[Oracle Conic SPA]
\label{lemma: orconspa}
    Let $\tilde X = X + E = V G + E \in \reals^{N \times T}_+$, where
        $V \in \reals^{N \times R}_+$ has rank $R$
        and $G = [I\ M] \Pi \in \reals^{R \times T}_+$ for some permutation matrix $\Pi$ and $M \in \reals^{R \times T-R}_+$.
        Denote $e_j, v_j$ the $j$th column of $E$ and $V$, respectively.
    Let $\| e_j \|_1 < \epsilon$ for all $j$ and suppose we know some $t>0$ such that $t + \epsilon < \| V \|_{-1, col}$.
    Assume
    \begin{equation}
    \begin{aligned}
    \label{eq: conic-spa-tolerance}
        \frac{\epsilon}{t}
        <
        \frac{C_1 \sigmin(V D_V)}{2 \sqrt{R} \cond(V D_V)^2 }.
    \end{aligned}
    \end{equation}
    Let $J$ be the index set extracted by OrConSPA with inputs $\tilde X, R, t$. Then there exists a permutation $P$ of $\{1, \hdots, R\}$ such that
    \begin{align}
        \max_{1 \leq j \leq r}\ \| (\tilde X D_{\tilde X})[:,J(j)] - (VD_V)[:,P(j)] \|_2 & \nonumber \\
         \leq
        2 C_2  (\epsilon t^{-1}) \cond(V D_V)^2, \text{ and } \label{eq: conic-spa-scaled-estimate}
    \end{align}
    \begin{align}
        \max_{1 \leq j \leq r}\ \| \tilde X[:,J(j)] - V[:,P(j)] \|_2 & \nonumber \\
       \quad  \leq
        (\| V \|_{1, col} + \epsilon) 2 C_2  (\epsilon t^{-1}) \cond(V D_V)^2 . \label{eq: conic-spa-full-estimate}
    \end{align}
\end{lemma}

\begin{proof}
    Let $v_j$ be the $j$th column of $V$, and allow the same notation for all other matrices.
    Since the oracle has given us $t \leq \| V \|_{-1} - \epsilon_1$ (where $t > 0$), we know that any column $j$ with $\| \tilde x_j\|_1 < t$ cannot satisfy $x_j = v_i$ for any $i$.
    Therefore, we can define $Y \in \reals^{T \times T}$ as a diagonal matrix given by
    \begin{align*}
        Y_{jj} =
        \begin{cases}
        0 & \text{ if } \| \tilde x_j \|_1 < t \\
        \| \tilde x_j \|_1^{-1} & \text{ otherwise}
        \end{cases}.
    \end{align*}
    Now define $\tilde X' = \tilde X Y$.
        Consider the matrices $V' = V D_V$, $X' = X D_X = V' G'$, and $G' = D_V^{-1} H D_X$.
    It must be that the columns of $G'$ sum to one, since the columns of $X'$ and $V'$ sum to one.
    Furthermore, columns $g_j$ of $G$ with only a single non-zero entry will only be scaled by some scalar $\alpha$, so $G' = [I\ M'] \Pi$ for some $M' \in \reals^{N \times T-R}_+$ \vtwominor{and permutation matrix $\Pi$}.
    Therefore, $X', V', G'$ all satisfy the conditions of Lemma \ref{lemma: spa-noisy};
    if we wish to apply SPA to identify $V'$ from $\tilde X' = X' + E'$, all that remains is to bound the noise $E' = \tilde X' - X'$.
    To bound the noise, we note that for any non-zero column of $\tilde X'$, we have
        $e'_j
         = \tilde x'_j - x'_j
         = \frac{x_j + e_j}{\| x_j + e_j \|_1} - \frac{x_j}{\| x_j \|_1}
         = \frac{e_j}{\| x_j + e_j \|_1} - \left( 1 - \frac{\|x_j\|_1}{\|x_j + e_j \|_1} \right) \frac{x_j}{\| x_j \|_1}$,
    and therefore we have the following bound for $\epsilon' =
        \| e'_j \|_2$:
    \begin{align*}
        \epsilon'
        &= \left\| \frac{e_j}{\| x_j + e_j \|_1}
            - \left( 1 - \frac{\|x_j\|_1}{\|x_j + e_j \|_1} \right) \frac{x_j}{\| x_j \|_1} \right\|_2
            \\
        &\leq \frac{ \| e_j\|_2 }{\| x_j + e_j \|_1}
            +  \left( 1 - \frac{\|x_j\|_1}{\|x_j + e_j \|_1} \right) \frac{\| x_j \|_2}{\| x_j \|_1}
            \\
        &\leq \frac{ \| e_j\|_2 }{\| x_j + e_j \|_1}
            +  \left( \frac{\|x_j\|_1 + \|e_j\|_1 - \|x_j\|_1}{\|x_j + e_j \|_1} \right) \frac{\| x_j \|_2}{\| x_j \|_1}
            \\
        &\leq \frac{ \| e_j\|_2 }{\| x_j + e_j \|_1}
            +   \frac{\|e_j\|_1}{\|x_j + e_j \|_1}  \leq \frac{2\epsilon}{t}.
    \end{align*}
    Concerning the zero columns, they cannot be noisy versions of a column of $V$ and therefore setting them to zero is irrelevant to the performance of SPA.
    Therefore, if we apply SPA to $X'$, then Lemma \ref{lemma: spa-noisy} and (\ref{eq: conic-spa-tolerance}) tells us our assumptions are satisfactory to recover a noisy estimate of $V'$ with error
        $C_2 (2 \epsilon t^{-1}) \cond(V D_V)^2$.
    Scaling each column by $D_{\tilde X}^{-1}$ gives the error in (\ref{eq: conic-spa-full-estimate}).
\end{proof}

Tighter bounds are achieved as $t$ approaches $\|V\|_{-1} - \epsilon$.
In real applications, this oracle is not available and $t$ is unknown.
However, we do know that if such a $t$ exists, then $0 \leq t \leq \| \tilde X \|_1$ and that $t = \| X_{:j} \|$ for some~$j$.
Since SPA is very fast (essentially $2NTR$ flops), it is reasonable to run OrConSPA for multiple values of $t$ and be certain that one of the outputs locates the columns of $V$ (that is, if such a value of $t$ exists at all).
In practice, a measure like mean square error can be used to select the ``best'' value of $t$.
The previous proof \vtwominor{specifically uses} SPA to the thresholded and scaled matrix $\tilde X' = \tilde X Y$.
However, there is nothing special about our choice of separable NMF algorithm.
In fact, other algorithms such as SNPA~\cite{gillis2014successive} or AnchorWords~\cite{arora2013practical} will yield more robust noise tolerances.
For a more complete list of separable NMF algorithms, see~\cite{gillis2019separable}.
In this vein, we can generalize Lemma \ref{lemma: orconspa} to an arbitrary separable NMF algorithm.

\begin{lemma}[Oracle Conic Location]
Let $\tilde X = X + E = VG + E \in \reals^{N \times T}_+$, where $V \in \reals^{N \times R}_+$ and $G = [I\ M] \Pi \in \reals^{R \times T}_+$ for some permutation matrix $\Pi$ and $M \in \reals^{R \times T-R}_+$.

Consider some arbitrary separable NMF algorithm \verb_ALG_ with a maximum input noise tolerance $\epsilon_{\text{in}} < m_1(V)$ and a maximum output noise $m_2(V, \epsilon_{\text{in}})$.
Furthermore, suppose $V D_V$ satisfies any other conditions of \verb_ALG_ (e.g. a condition on the rank).
Denote $e_j, v_j$ the $j$th column of $E$ and $V$, respectively.
Now suppose $\|e_j\|_1 \leq \epsilon$ for all $j$ and suppose we know of some $t > 0$ such that $t + \epsilon < \| V \|_{-1}$ and
    $\frac{\epsilon}{t}
        <
    \frac{1}{2} m_1(V D_V)$.
Let $J$ be the index set extracted by OCL with inputs $\tilde X, R, t$, and \verb_ALG_.
Then there exists a permutation $P$ of $\{1, \hdots, R\}$ such that
    $\max_{j} \| \tilde X[:, J(j)] - V[:, P(j)] \|_2
        \leq
    (\| V \|_1 + \epsilon) m_2(V D_V, 2\epsilon t^{-1})$.

\end{lemma}

\subsection{Relevant Lemmas}

\subsubsection{Perturbation bounds for least squares}

\begin{lemma}[Least Squares Perturbation, \cite{lotstedt1983perturbation} Theorem 3, Remark 2]
\label{lemma: lotstedt}
    Consider full rank matrices $A, \tilde A \in \reals^{T \times R}$ and vectors $b, \tilde b \in \reals^{T}$.
    Define $F = A - \tilde A$ and $e = b - \tilde b$
    Suppose there is a convex set $\calC$ and a vector $y \in \calC$ such that $Ay = b$ and let
        $\tilde y = \argmin_{y' \in \calC} \| \tilde A y' - \tilde b \|_2^2$.
    If $\sigmin(A) > \sigmax(F)$, then
    \begin{align*}
        \| y - \tilde y \|_2
        \leq
        \frac{\sigmax(F) \| y \|_2 + \| e \|_2 }{ \sigmin(A) - \sigmax(F) }.
    \end{align*}
\end{lemma}

\begin{corollary}
\label{cor: perturbation}

    Suppose we are given $\tilde X, X \in \reals^{T \times N}$ such that $E = X - \tilde X$ satisfies $\| E[:,j] \| \leq \epsilon$ for all~$j$.
    Additionally, suppose we also have $\tilde V, V \in \reals^{N \times KL}$ such that $F = V - \tilde V$ satisfies $\sigmin(V) > \sigmax(F)$.
    Suppose there is a $G \geq 0$ such that $VG = X$ and let
    \begin{align*}
        \tilde G
            &= \argmin_{G' \geq 0} \| \tilde V G' - \tilde X \|_F^2.
    \end{align*}
    Also let $\gamma = \| G[:, j] \|_{2, col}$ be the maximum norm of any column of $G$.
    Then the distance between rows of $G^*$ and $\hat G$ is bounded by
        $\max_{i}\ \| G[i,:] - \hat G[i,:] \|_2
        \leq \sqrt{T} \left( \frac{\sigmax(F) \gamma + \epsilon }{\sigmin(V) - \sigmax(F)} \right)$.
\end{corollary}
\begin{proof}
    For the $j$th column $g_j$ of $G$, we have the bound
    \begin{align*}
        \| g_j - \tilde g_j \|_2^2
        \leq
        \left( \frac{\sigmax(F) \| g_j \|_2 + \epsilon }{\sigmin(V) - \sigmax(F)} \right)^2.
    \end{align*}
    Then for a row $G[i, :]$ we must have
    \begin{align*}
        \| G[i, :] - \tilde G[i, :] \|_2^2
        \leq
        T \left(\frac{\sigmax(F) \gamma + \epsilon }{\sigmin(V) - \sigmax(F)}\right)^2.
    \end{align*}
    Taking the square root yields the desired bound.
\end{proof}

\subsubsection{Cosine Similarity}




\begin{lemma}
    \label{lemma: angle-perturb}
    Consider the vectors $y_1, y_2, e_1, e_2 \in \reals^T$.
    Let $\theta = \cos(y_1, y_2) > 0$.
    Suppose $\|e_1\|_1 \leq \epsilon \| y_1 \|_1$ and $\|e_2\|_2 \leq \epsilon \|y_2 \|_2$ for $0 < \epsilon < 1$.
    Then we have
    \begin{align*}
        \frac{\theta - 3\epsilon}{(1 + \epsilon)^2}
        \leq \cos(y_1 + e_1, y_2 + e_2) \leq
        \frac{\theta + 3\epsilon}{(1 - \epsilon)^2}.
    \end{align*}
\end{lemma}

\begin{proof}
    First consider when $\|y_1\|_2 = \| y_2\|_2 = 1$.
    We have
        $\frac{ (y_1+e_1)^T(y_2+b_2) }{\| y_1 + e_1 \|_2 \| y_2 + b_2 \|_2}
        = \frac{\theta + y_1^Te_2 + y_2^T e_1 + e_1^T e_2}{\| y_1 + e_1 \|_2 \| y_2 + e_2 \|_2}
        \geq \frac{\theta - 2\epsilon - \epsilon^2}{\| y_1 + e_1 \|_2 \| y_2 + e_2 \|_2 }
        \geq \frac{\theta - 3\epsilon}{(1+\epsilon)^2}$.
    Now consider when $\|y_1\|_2, \|y_2\|_2 \neq 1$.
    Since cos is scale invariant, we scale the vectors to $(y_1 + e_1) / \|y_1\|_2$ and $(y_2 + e_2) / \|y_2\|_2$ and apply the above statement to obtain the  lower bound in our lemma.
    The upper bound follows similarly.
\end{proof}

\begin{corollary}
    \label{cor: angle-linear}
    Under the same assumptions as in Lemma \ref{lemma: angle-perturb} and $\epsilon < \frac{1}{2}$, we have
    \begin{align*}
        \theta - 35 \epsilon \leq
        \cos(y_1 + e_1, y_2 + e_2)
        \leq \theta + 35 \epsilon.
    \end{align*}
\end{corollary}

\begin{proof}
    For the lower bound, use the Taylor expansion of $(1+\epsilon)^{-1}$ to show that
    \begin{align*}
        \frac{\theta - 3\epsilon}{(1 + \epsilon)^2}
        &= (\theta - 3\epsilon) \left[ 1 - \epsilon + \epsilon^2 - \epsilon^3 + \hdots \right]^2
        \\
        &\geq (\theta - 3\epsilon) (1 - \epsilon)^2
             \text{ (since $\epsilon^{2n} - \epsilon^{2n+1} > 0$)}
        \\
        &\geq (\theta - 3\epsilon) (1 - 2\epsilon) \geq \theta - 5\epsilon.
    \end{align*}
    For the upper bound, consider the Taylor expansion of $(1 - \epsilon)^{-1} =
    \left[1 + \epsilon + \epsilon^2 + \hdots \right]
    =
    \left[1 + \epsilon(1 + \epsilon + \hdots) \right]
    =
    \left[1 + \epsilon \frac{1}{1 - \epsilon} \right]
    \geq (1 + 2\epsilon),
    $
    hence
    $
    \frac{\theta + 3\epsilon}{(1 - \epsilon)^2}
    \leq (\theta + 3\epsilon) (1 + 2\epsilon)^2 \leq (\theta + 3\epsilon) (1 + 8\epsilon) \leq \theta +  35 \epsilon$.
\end{proof}

\subsubsection{Clustering and Sorting}

\begin{lemma}[Clustering] \label{lemma: clustering}
    Suppose $G \in \reals^{KL \times T}$ is produced from (\ref{eq: block convsum}) for a CNMF problem with some $H$ that satisfies [$\delta$-sequentially unique].
    Consider $\tilde G$ such that
        $\max_{1 \leq j \leq KL} \ \| g_{j} - \tilde g_{\permj} \|_2
        = \epsilon
        < \frac{\delta}{70} \| G \|_{-2, row}$,
    for some permutation $P$, 
    where $g_j = G[j,:]$ and $\tilde g_j = \tilde G[j,:]$.
    Then in $O(TK^2L^3)$ time we can recover $K$ sets $\calC_k$ of $L$ indices such that for each row $h_k$ of $H$, there is some group $\calC_{k}$ where for all $\ell = 0, \hdots, L-1$, there exists some $j \in \calC_{k}$ such that $g_{P^{-1}(j)} = S_\ell^T h_k$.
\end{lemma}

\begin{proof}
    Fix some $k$.
    Consider $i, j$ such that $g_i = S_{\tau_1}^T h_k$ and $g_j = S_{\tau_2} h_k$ for $\tau_1, \tau_2 \in \{0, \hdots L-1\}$.
    Then we must have one of $g_i = S_\ell g_j$ or $g_j = S_\ell g_i$ for some $\ell \in \{0, \hdots, L-1\}$.
    Without loss of generality, assume $g_i = S_\ell g_j$.
    Let $r_i = \tilde g_{\permi} - g_{i}$ and $r_j = S_\ell^T \tilde g_{\permj} - S_\ell^T g_j$.
    By assumption, $\max(\| r_j \|, \| r_i \|) < (\delta / 70) \|G \|_{-2, row}$.
    Now we apply Corollary \ref{cor: angle-linear} to $\tilde g_i = g_i + r_i$ and $S_\ell^T \tilde g_j = S_\ell^T g_j + r_j$ and consider that $\cos(g_i, S_\ell^T g_j)=1$ to obtain
        $\cos(S_\ell^T \tilde g_j, \tilde g_i) > 1 - \frac{\delta}{2}$,
    and hence $\cos_L(\tilde g_j, \tilde g_i) \geq 1 - \frac{\delta}{2}$.

    Now consider $i, j$ such that $g_i = S_{\tau_1}^T h_k$ for some $\tau_1 \in \{0, \hdots, L-1 \}$, but $g_j \neq S_{\tau_2}^T h_k$ for any $\tau_2 \in \{0, \hdots, L-1 \}$.
    Then by [$\delta$-sequentially unique], we must have that $\max(\cos(g_i, S_\ell^T g_j), \cos(S_\ell^T g_i, g_j)) \leq 1 - \delta$ for any $\ell \in \{0, \hdots, L-1\}$.
    Applying Corollary \ref{cor: angle-linear} with the same construction as before, we have
    \begin{align*}
        \max(\cos(S_\ell^T \tilde g_j, \tilde g_i), \cos(\tilde g_j, S_\ell^T \tilde g_i)) < 1 - \frac{\delta}{2},
    \end{align*}
    for any $\ell \in \{0, \hdots, L-1\}$.
    Therefore, $\cos_L(\tilde g_1, \tilde g_2) < 1 - \frac{\delta}{2}$.
    Together, these two angle bounds gives a strict criterion for clustering the rows of $G$ into $K$ groups of $L$ vectors, according to their corresponding row of $H$.
    The determine the runtime, we note that we must compute the angle between $KL$ vectors $L$ times each.
    Computing the angle of each vector requires $\order(T)$ flops, so in total we require $\order(TK^2L^3)$ operations.
    The grouping of the vectors can be done simply in $\order(K^2L)$ and is comparatively negligible as $K < T$.
\end{proof}




\begin{lemma}[Sorting] \label{lemma: sorting}
    Suppose $G \in \reals^{KL \times T}$ is produced from (\ref{eq: block convsum}) for a CNMF problem with some $H$ that satisfies [$\delta$-sequentially unique].
    Suppose there exists a permutation $P$ 
    and a matrix $\tilde G$ such that
    \begin{align*}
        \max_{1 \leq j \leq KL} \ \| g_{j} - \tilde g_{\permj} \|_2
        = \epsilon
        < \frac{\delta}{70} \| G \|_{-2, row},
    \end{align*}
    where $g_j = G[j,:]$, $\tilde g_j = \tilde G[j,:]$.

    Now suppose we have some set of $L$ indices $\calC_k$, and that for all $\ell = 0, \hdots, L-1$ there exists $j \in \calC_k$ such that $g_{P^{-1}(j)} = S_\ell^T h_k$.

    Then in $\order(TL^3)$ time we can recover a bijective map $\pi: \calC_k \rightarrow \{0, \hdots, L-1\}$ such that $g_{P^{-1}(j)} = S_{\pi(j)}^T h_k$.
\end{lemma}

\begin{proof}
    For any $i, j \in \calC_k$, a nearly identical argument to Lemma \ref{lemma: clustering} gives us a decision criterion for finding a unique $\ell_{ij} \in \{0, \hdots, L-1\}$ such that either $g_{P^{-1}(j)} = S_{\ell_{ij}}^T g_i$ or $g_{P^{-1}(i)} = S_{\ell_{ij}}^T g_j$.
    We find such $\ell$ for each pair of indices, and construct a mapping $y(i, j) : \calC_k \times \calC_k \rightarrow \{-L+1, \hdots, L-1 \}$ defined by
    \begin{align*}
        y(i, j) =
        \begin{cases}
            \ell_{ij} & \text{ if } g_{P^{-1}(i)} = S_{\ell_{ij}}^T g_j \\
            -\ell_{ij} & \text{ if } g_{P^{-1}(j)} = S_{\ell_{ij}}^T g_i
        \end{cases}.
    \end{align*}
    By the construction of $G$ and $\calC_k$, there is only one $i$ such that $y(i, j)$ will be strictly nonnegative for all $j \in \calC_k$.
    Then our permutation is $\pi$ is given by $\pi(j) = y(i, j)$.

    Each comparison requires $L$ distance computations, and each distance computations takes $\order(T)$ flops. Since we must measure the distance between $L$ different vectors, this totals to $\order(TL^3)$ flops.
    The construction of $y$ and $\pi$ requires $\order(L^2)$ flops and is comparatively negligible.
\end{proof}

\subsection{Proof of the Recovery Guarantee in the Presence of Noise}
\label{apdx: noisy-proof}

Lemma \ref{lemma: noisy-recovery-full} proves the recovery guarantee in slightly more details than in Theorem~\ref{thm: noisy-recovery-abridged}.
It can then be used to obtain Theorem~\ref{thm: noisy-recovery-abridged}.

\begin{lemma}
\label{lemma: noisy-recovery-full}
    Suppose a CNMF problem with inputs $\tilde X, L, K$ is convolutive separable with respect to $\delta, \epsilon, A$.
    Let $V' = V A D_{VA}$ and $G' = D_{VA}^{-1} A^{-1} G$
    and suppose we know some $t>0$ such that
    \begin{align}
        \label{eq: noise-bound-additive}
        \epsilon + t
        &<
        \| V A \|_{-1, col} ,
            \\
        \label{eq: small-noise}
        \frac{\epsilon}{t}
        &<
        \frac{\sigmin(V')}{2 \sqrt{KL} \cond(V')^2}
        \min\left(
            C_1,
            C_2^{-1}
        \right)
            \\
        \label{eq: big-dist}
        \delta
            &>  \frac{ 70 \sqrt{T} }{ \|G'\|_{-2,row} } \left( \frac{c \|G'\|_{2, col} + \epsilon }{\sigmin(V') - c} \right),
    \end{align}
    where $C_1$ and $C_2$ are universal constants independent of all other terms,
        $V$ is defined in (\ref{eq: block convsum}),
        and $c =  2C_2 (\epsilon t^{-1}) \cond(V')^2 \sqrt{KL}$.
    Then in $\order(NTKL + TK^2L^3)$ time and the time for one NNLS solve, Algorithm \ref{alg: LECS} finds $\tilde H \in \reals^{T \times K}$ with bounded error, in the sense that there exists some permutation $P$ such that
    \begin{align}
    \label{eq: recovery-error}
        \min_{1 \leq j \leq K} \ \cos\left( h_j, \tilde h_\permj \right)
        \geq
        1 - \frac{\delta}{2},
    \end{align}
    where $h_j, \hat h_\permj$ are the $j$th and $\permj$th rows of $H$ and $\hat H$, respectively.
\end{lemma}

\textit{Proof}. We break the proof up into several steps.

\paragraph{Location}
We know that $V = X[:, \calC] A \Pi$ for some index set $\calC$, some diagonal scaling matrix $A$, and some permutation matrix $\Pi$.
Instead of identifying $V$ directly, we identify $V A = X[:, \calC] \Pi$, which satisfies $X = VA [I\ M] \Pi'$ for some $M \in \reals^{N \times T-R}$ and some permutation matrix $\Pi'$.
In particular, we apply OrConSPA to inputs $\tilde X, R, t$ to recover $V' = VA D_{VA}$.
Since $X = VA [I\ M] \Pi'$ and the noise is bounded in (\ref{eq: noise-bound-additive}) and (\ref{eq: small-noise}) by
    $\epsilon + t
    \leq
    \| VA \|_{-1}$ and
    $\frac{\epsilon}{t}
    \leq
    \frac{C_1 \sigmin(V')}{2 \sqrt{KL} \cond(V')^2 }$,
we can apply Lemma \ref{lemma: orconspa} and show that the index set $J$ output by OrConSPA satisfies
    \[ \max_{j}\
    \| (\tilde X D_{\tilde X})[:, J(j)] - V'[:, P(j)] \|_2^2
    \leq
    2 C_2 \frac{\epsilon}{t} \cond(V')^2. \]
Let $\tilde V = (\tilde X D_{\tilde X})[:, J(j)]$.

\paragraph{Perturbation}
Given $\tilde V$, we estimate $G' = D_{VA}^{-1} A^{-1} G$ using nonnegative least squares.
Let $F = \tilde V - V'$.
We know the largest singular value is bounded above by the Frobenius norm, so we have
    $\sigmax(F) \leq \sqrt{KL} \left( 2 C_2 (\epsilon t^{-1}) \cond(V')^2 \right) = c$.
Rearranging (\ref{eq: small-noise}) using the latter term of the minimum gives us the bound
    $\sigmin(V') > 2 C_2 (\epsilon t^{-1}) \cond(V')^2  \sqrt{KL} = c$.
This allows us to apply Corollary \ref{cor: perturbation}, so the recovered matrix $\tilde G$ must satisfy
    $\max_{i}\ \| G'[i,:] - \tilde G[i,:] \|
        \leq \sqrt{T} \left( \frac{\sigmax(F) \| G' \|_{2, col} + \epsilon }{\sigmin(V') - \sigmax(F)} \right)$.
We then have
\begin{align}
\label{eq: G-perturb-bound}
    \max_{1 \leq i \leq R}\ \| G'[i,:] - \tilde G[i,:] \|
        \leq
    \sqrt{T} \left( \frac{c \| G' \|_{2, col} + \epsilon }{\sigmin(V') - c} \right).
\end{align}

\paragraph{Clustering and Sorting}
Since $G' = D_{VA}^{-1} A^{-1} G$, the angle to a row of $G'$ is the same as an angle to a row of $G$.
Rearranging (\ref{eq: big-dist}) we have
\begin{align*}
    \frac{\delta \| G' \|_{-2, row}}{70} > \sqrt{T} \left( \frac{c \| G' \|_{2, col} + \epsilon }{\ssigmin(VA) - c} \right).
\end{align*}
Combining this with (\ref{eq: G-perturb-bound}) lets us apply Lemma \ref{lemma: clustering} and Lemma \ref{lemma: sorting} to achieve the correct clustering and sorting.
The bound (\ref{eq: recovery-error}) follows automatically.
\hfill $\qed$

\begin{proof}[Proof of Theorem \ref{thm: noisy-recovery-abridged}]
   All we must do is simplify (\ref{eq: small-noise}) and (\ref{eq: big-dist}) to match the bounds in Theorem \ref{thm: noisy-recovery-abridged}.
   For (\ref{eq: small-noise}), simply set
       $C_a = \frac{\min(C_1, C_2^{-1})}{2}$.
   For (\ref{eq: big-dist}), we begin by recalling our assumption that $\| VA \|_{2, col} > 1$ and that $t > 1$.
   We must have that $\| A^{-1} G \|_{-2, row} > 1$ since each row has an entry with a 1 in it by construction.
   Then since $\| VA \|_{2, col} > 1$ and equivalently $(D_{VA})^{-1}_{ii} > 1$, then $\| G' \|_{-2, row} > 1$.
   Given these assumptions, we have the sufficient condition
   \begin{align*}
       \delta >
       \frac{70 \sqrt{T} (c \| G \|_{2, col} + \epsilon)}{\sigmin(V') - c}.
   \end{align*}
   Dividing the numerator and denominator by $2 C_2 \cond(V')^2 \sqrt{KL}$ leads to
       $\delta >
       \frac{70 \sqrt{T} ( \| G' \|_{2, col} (\epsilon t^{-1} ) + \epsilon)}{\frac{\sigmin(V')}{2 C_2 \cond(V')^2 \sqrt{KL}} -  (\epsilon t^{-1}) }$.
   Then we have the sufficient condition
   \begin{align*}
       \delta >
       \frac{70 \sqrt{T} ( \| G' \|_{2, col} (\epsilon t^{-1}) + \epsilon) )}{ \rho },
   \end{align*}
   where $\rho = C_a \frac{\sigmin(V')}{ \cond(V')^2 \sqrt{KL}} -  (\epsilon t^{-1})$. Since we assumed $t > 1$,
       $\delta >
       \frac{140 \sqrt{T}  \| G' \|_{2, col} \epsilon }{ \rho }$
   is also sufficient. Setting $C_b^{-1} = 140 $ gives us $\delta > C_b^{-1} \sqrt{T} \|G'\|_{2, col} \epsilon / \rho$ which in turn reduces to $\epsilon < C_b \delta \rho / \| G' \|_{2, col} \sqrt{T}$.
\end{proof}


\section*{Acknowledgment}

The authors thank the reviewers for their feedback which helped improve the paper significantly.

\newpage

\bibliographystyle{spmpsci}
\bibliography{refs}

\end{document}